\documentclass{article} 
\usepackage{iclr2026_conference,times}

\usepackage{amsmath,amsfonts,bm}

\def\eqref#1{equation~\ref{#1}}

\def\1{\bm{1}}

\DeclareMathAlphabet{\mathsfit}{\encodingdefault}{\sfdefault}{m}{sl}
\SetMathAlphabet{\mathsfit}{bold}{\encodingdefault}{\sfdefault}{bx}{n}

\usepackage[utf8]{inputenc} 
\usepackage[T1]{fontenc}    
\usepackage{hyperref}       
\usepackage{url}            
\usepackage{booktabs}       
\usepackage{amsfonts}       
\usepackage{nicefrac}       
\usepackage{microtype}      
\usepackage{xcolor}         
\usepackage{caption}

\usepackage{amsmath}
\usepackage{amssymb}
\usepackage{mathtools}
\usepackage{amsthm}
\usepackage{tcolorbox}
\usepackage{wrapfig}
\usepackage{algorithm}
\usepackage{algorithmic}

\theoremstyle{plain}
\newtheorem{theorem}{Theorem}[section]

\newtheorem{lemma}[theorem]{Lemma}
\newtheorem{corollary}[theorem]{Corollary}
\theoremstyle{definition}
\newtheorem{definition}[theorem]{Definition}
\newtheorem{assumption}[theorem]{Assumption}
\theoremstyle{remark}

\graphicspath{{figures/}}

\definecolor{myblue}{HTML}{0BBCD6}
\definecolor{myred}{HTML}{EF3E4A}
\definecolor{citecolor}{RGB}{56,89,197}
\definecolor{linkcolor}{RGB}{137,158,167}
\hypersetup{
	colorlinks=true,
	citecolor=citecolor,
	linkcolor=linkcolor
}

\title{Representation Convergence: Mutual Distillation is Secretly a Form of Regularization}

\author{Zhengpeng Xie$^1$, ~Jiahang Cao$^1$, ~Changwei Wang$^2$, ~Fan Yang$^3$, ~Marco Hutter$^3$ \\
\bfseries Qiang Zhang$^1$, ~Jianxiong Zhang$^4$, ~Renjing Xu$^1$\thanks{Corresponding author. First author: \texttt{zhengpengxie00@gmail.com}.} \\
$^1$HKUST(GZ) ~$^2$QLU ~$^3$ETH Zurich ~$^4$SCU\\
}

\iclrfinalcopy 
\begin{document}

\maketitle

\begin{abstract}
In this paper, we argue that mutual distillation between reinforcement learning policies serves as an \textit{implicit regularization}, preventing them from overfitting to irrelevant features. We highlight two \textit{separate} contributions: (i) Theoretically, for the first time, we prove that enhancing the policy robustness to irrelevant features leads to improved generalization performance. (ii) Empirically, we demonstrate that mutual distillation between policies contributes to such robustness, enabling the spontaneous emergence of \textit{invariant representations} over pixel inputs. Ultimately, we do not claim to achieve state-of-the-art performance but rather focus on uncovering the underlying principles of generalization and deepening our understanding of its mechanisms. Our website: \url{https://dml-rl.github.io/}.
\end{abstract}

\section{Introduction}
Humans exhibit a remarkable ability to learn robustly and generalize across diverse environments. Once a skill is acquired, it often transfers seamlessly to new contexts that share the same underlying semantics, even when their visual appearance differs substantially. For example, consider a person who becomes proficient at a video game, even if the background graphics or character textures are altered, the player retains their ability to perform well, effortlessly adapting to the new setting. This suggests that human learning is not overly dependent on low-level visual details, but rather grounded in abstract representations that capture the essential structure of a task. Neuroscientific studies support this view, linking abstract reasoning to the human prefrontal cortex \citep{bengtsson2009representation, dumontheil2014development}, and highlighting the role of inhibitory neurons in enhancing cognitive processing efficiency \citep{pi2013cortical}.

In stark contrast, visual reinforcement learning (VRL) agents often struggle with generalization. While they can be trained to solve complex tasks in specific environments, even minor changes, such as shifts in color schemes or background textures, can significantly degrade their performance. This sensitivity indicates that VRL agents tend to overfit to superficial visual features, failing to capture the underlying structure of the task \citep{cobbe2019quantifying, cobbe2020leveraging}. These limitations give rise to a fundamental question:
\begin{center}
	\textit{What hinders reinforcement learning agents from generalizing like humans? How can we enable them to learn robust representations that drive human-like generalization behavior?}
\end{center}

The core reason behind the limited generalization ability of VRL agents lies in their reliance on convolutional neural networks (CNNs) as visual encoders. While CNNs are the de facto choice for processing high-dimensional visual inputs, they are notoriously sensitive to even small perturbations \citep{goodfellow2014explaining}. This brittleness significantly hampers the robustness of learned policies and limits their ability to generalize across visually diverse environments. To address this issue, one common strategy is to apply data augmentation \citep{shorten2019survey}, which improves robustness by diversifying the training distribution and reducing dataset-induced biases. Alternatively, invariant representation learning has emerged as a principled approach to tackle generalization problem from a feature-learning perspective. It aims to extract representations that remain stable under a wide range of input transformations, thereby promoting robustness and transferability \citep{nguyen2021domain}.

While data augmentation is an effective bias mitigation technique, its reliance on task-specific strategies that are manually crafted by human experts, poses a challenge for designing task-independent solutions. In contrast, our method enables agents to generalize without any handcrafted augmentations or external priors, relying purely on training experience. Invariant representation learning is a promising approach to enhance model's cross-domain generalization. However, it relies on transformation correspondences, which are fundamentally inaccessible in the generalization scenarios of reinforcement learning due to the dynamic nature of environments. In addition, invariant representation framework inherently separates the encoder from the model, unnecessarily complicating the theoretical analysis. Instead, our framework is theoretically and empirically end-to-end.

In this paper, we first propose a novel theoretical framework to analyze the generalization problem in reinforcement learning and show that the policy robustness to irrelevant features enhances its generalization performance. Building upon this principled insight, we then provide empirical evidence that deep mutual learning (DML) \citep{zhang2018deep} can implicitly prevent them from overfitting to such irrelevant features, leading to robust learning process and significant generalization improvements.

In summary, the main contributions of this paper are as follows:
\begin{itemize}
	\item We theoretically prove that improving the policy robustness to irrelevant features enhances its generalization performance. To the best of our knowledge, we are the first to provide a rigorous proof of this intuition.
	\item We propose a hypothesis that deep mutual learning (DML) enhances the generalization performance of the policy by implicitly regularizing irrelevant features. We also provide intuitive insights to support this hypothesis.
	\item Strong empirical results support our theory and hypothesis, showing that DML technique leads to consistent improvements in generalization performance.
\end{itemize}

\section{Preliminaries}
In this section, we introduce reinforcement learning under the generalization setting in Section \ref{subsec:mdp}, as well as the DML technique in Section \ref{subsec:dml}.

\subsection{Markov Decision Process and Generalization}
\label{subsec:mdp}
Markov decision process (MDP) is a mathematical framework for sequential decision-making, which is defined by a tuple $\mathcal{M}=\left(\mathcal{S},\mathcal{A},r,\mathcal{P},\rho,\gamma\right)$, where $\mathcal{S}$ and $\mathcal{A}$ represent the state space and action space, $r:\mathcal{S}\times\mathcal{A}\mapsto\mathbb{R}$ is the reward function, $\mathcal{P}:\mathcal{S}\times\mathcal{A}\times\mathcal{S}\mapsto[0,1]$ is the dynamics, $\rho:\mathcal{S}\mapsto[0,1]$ is the initial state distribution, and $\gamma\in(0,1)$ is the discount factor.

Define a policy $\mu:\mathcal{S}\times\mathcal{A}\mapsto[0,1]$, the action-value function and value function are defined as
\begin{equation}
	Q^{\mu}(s_t,a_t)=\mathbb{E}_{\mu}\left[\sum_{k=0}^{\infty}\gamma^k r(s_{t+k},a_{t+k})\right],\enspace V^{\mu}(s_t)=\mathbb{E}_{a_t\sim\mu(\cdot|s_t)}\left[Q^{\mu}(s_t,a_t)\right].
\end{equation}
Given $Q^{\mu}$ and $V^{\mu}$, the advantage function can be expressed as $A^{\mu}(s_t,a_t)=Q^{\mu}(s_t,a_t)-V^{\mu}(s_t)$.

In our generalization setting, we introduce a rendering function $f:\mathcal{S}\mapsto\mathcal{O}_f\subset\mathcal{O}$ to obfuscate the agent's actual observations, which is a \textit{bijection}\footnote{We define $\mathcal{O}_f:=\left\{f(s)\vert s\in\mathcal{S}\right\}$, which means for any $s_1\neq s_2$, we have $f(s_1)\neq f(s_2)$.} from $\mathcal{S}$ to $\mathcal{O}_f$. We now define the MDP induced by the underlying MDP $\mathcal{M}$ and the rendering function $f$, denote it as $\mathcal{M}_f=\left(\mathcal{O}_f,\mathcal{A},r_f,\mathcal{P}_f,\rho_f,\gamma\right)$, where $\mathcal{O}_f$ represents the observation space, $r_f:\mathcal{O}_f\times\mathcal{A}\mapsto\mathbb{R}$ is the reward function, $\mathcal{P}_f:\mathcal{O}_f\times\mathcal{A}\times\mathcal{O}_f\mapsto[0,1]$ is the dynamics, and $\rho_f:\mathcal{O}_f\mapsto[0,1]$ is the initial observation distribution. We present the following assumptions:
\begin{assumption}\label{assumption 1}
	Assume that $f$ can be sampled from a distribution $p:\mathcal{F}\mapsto[0,1]$, where $f\in\mathcal{F}$.
\end{assumption}

\begin{assumption}\label{assumption 2}
	Given any $f\in\mathcal{F}$, $o_0^f,o_t^f,o_{t+1}^f\in\mathcal{O}_f$ and $a_t\in\mathcal{A}$, assume that $r_f(o_t^f,a_t)=r(f^{-1}(o_t^f),a_t),\mathcal{P}_f(o_{t+1}^f|o_t^f,a_t)=\mathcal{P}(f^{-1}(o_{t+1}^f)|f^{-1}(o_t^f),a_t),\rho_f(o_0^f)=\rho(f^{-1}(o_0^f))$.
\end{assumption}

\textbf{Explanation.} Assumption \ref{assumption 2} states that all $\mathcal{M}_f$ share a common underlying MDP $\mathcal{M}$, in which the agent's observations are perturbed by different rendering functions while all other components remain unchanged, much like different painters depicting the same scene in their own styles.

Next, consider an agent interacting with $\mathcal{M}_f$ following the policy $\pi:\mathcal{O}\times\mathcal{A}\mapsto[0,1]$ to obtain a trajectory
\begin{equation}
	\tau_f=(o_0^f,a_0,r_0^f,o_1^f,a_1,r_1^f,\dots,o_t^f,a_t,r_t^f,\dots),
\end{equation}
where $o_0^f\sim\rho_f(\cdot)$, $a_t\sim\pi(\cdot|o_t^f)$, $r_t^f=r_f(o_t^f,a_t)$ and $o_{t+1}^f\sim\mathcal{P}_f(\cdot|o_t^f,a_t)$, we simplify the notation to $\tau_f\sim\pi$. During training, the agent is only allowed to access a subset of all MDPs, which is $\left\{\mathcal{M}_f|f\in\mathcal{F}_{\mathrm{train}}\subset\mathcal{F}\right\}$, and then tests its generalization performance across all MDPs. Thus, denote $p_{\mathrm{train}}:\mathcal{F}_{\mathrm{train}}\mapsto[0,1]$ as the distribution over $\mathcal{F}_{\mathrm{train}}$, the agent's training performance $\eta(\pi)$ and generalization performance $\zeta(\pi)$ can be expressed as
\begin{equation}\label{training and generalization performance}
	\eta(\pi)=\mathbb{E}_{f\sim p_{\mathrm{train}}(\cdot),\tau_f\sim\pi}\left[\sum_{t=0}^{\infty}\gamma^tr_f(o_t^f,a_t)\right],\enspace\zeta(\pi)=\mathbb{E}_{f\sim p(\cdot),\tau_f\sim\pi}\left[\sum_{t=0}^{\infty}\gamma^tr_f(o_t^f,a_t)\right].
\end{equation}
The goal of the agent is to learn a policy $\pi$ that maximizes the generalization performance $\zeta(\pi)$.

\subsection{Deep Mutual Learning}
\label{subsec:dml}
Deep mutual learning (DML) 
\citep{zhang2018deep} is a mutual distillation technique in supervised learning. Unlike the traditional teacher-student distillation strategy, DML aligns the probability distributions of multiple student networks by minimizing the KL divergence loss during training, allowing them to learn from each other. Specifically,
\begin{equation}\label{dml loss}
	\mathcal{L}_{\mathrm{DML}}=\mathcal{L}_{\mathrm{SL}}+\alpha\mathcal{L}_{\mathrm{KL}},
\end{equation}
where $\mathcal{L}_{\mathrm{SL}}$ and $\mathcal{L}_{\mathrm{KL}}$ represent the supervised learning loss and the KL divergence loss, respectively, $\alpha$ is the weight. Using DML, the student cohort effectively pools their collective estimate of the next most likely classes. Finding out and matching the other most likely classes for each training instance according to their peers increases each student's posterior entropy, which helps them converge to a more robust representation, leading to better generalization.

\section{Theoretical Results} \label{Theoretical Results}
In this section, we present the main results of this paper, demonstrating that enhancing the agent's robustness to irrelevant features will improve its generalization performance.

A key issue is that we do not exactly know the probability distribution $p_{\mathrm{train}}$. Note that $\mathcal{F}_{\mathrm{train}}$ is a subset of $\mathcal{F}$, we naturally assume that the probability distribution $p_{\mathrm{train}}$ can be derived from the normalized probability distribution $p$.
\begin{assumption}\label{assumption 3}
	For any $f\in\mathcal{F}$, assume that
	\begin{equation}
		p_{\mathrm{train}}(f)=\frac{p(f)\cdot\mathbb{I}(f\in\mathcal{F}_{\mathrm{train}})}{Z},\enspace p_{\mathrm{eval}}(f)=\frac{p(f)\cdot\mathbb{I}(f\in\mathcal{F}_{\mathrm{eval}})}{1-Z},
	\end{equation}
	where $Z=\int_{\mathcal{F}_{\mathrm{train}}}p(f)\mathrm{d}f$ and $1-Z$ is the normalization term, $\mathcal{F}_{\mathrm{eval}}=\mathcal{F}-\mathcal{F}_{\mathrm{train}}$, $\mathbb{I}(\cdot)$ denotes the indicator function.
\end{assumption}

An interesting fact is that, for a specific policy $\pi$, if we only consider its interaction with $\mathcal{M}_f$, we can establish a bijection between this policy and a certain underlying policy that directly interacts with $\mathcal{M}$. We now denote it as $\mu_f(\cdot|s_t)=\pi(\cdot|f(s_t))$. By further defining the normalized discounted visitation distribution $d^{\mu}(s)=(1-\gamma)\sum_{t=0}^{\infty}\gamma^t\mathbb{P}(s_t=s|\mu)$, we can use this underlying policy $\mu_f$ to replace the training and generalization performance of the policy $\pi$. Specifically, we have the following connection:
\begin{lemma}\label{lemma}
	For any given policy $\pi$, define its underlying policy as $\mu_f(\cdot|s_t)=\pi(\cdot|f(s_t))$, then
	\begin{equation}
		\eta(\pi)=\frac{1}{1-\gamma}\mathop{\mathbb{E}}_{\substack{f\sim p_{\mathrm{train}}(\cdot)\\s\sim d^{\mu_f}(\cdot)\\a\sim\mu_f(\cdot|s)}}\left[r(s,a)\right],\enspace\zeta(\pi)=\frac{1}{1-\gamma}\mathop{\mathbb{E}}_{\substack{f\sim p(\cdot)\\s\sim d^{\mu_f}(\cdot)\\a\sim\mu_f(\cdot|s)}}\left[r(s,a)\right].
	\end{equation}
\end{lemma}
\begin{proof}
	See Appendix \ref{proof lemma}.
\end{proof}

We can thus analyze the generalization problem using the underlying policy $\mu_f$. Then, define $L_{\pi}$ as the first-order approximation of $\eta$ \citep{schulman2015trust}, we can derive the following lower bounds:
\begin{theorem}[Training performance lower bound] \label{theorem 1}
	Given any two policies, $\tilde{\pi}$ and $\pi$, the following bound holds:
	\begin{equation}
		\eta(\tilde{\pi})\geq L_{\pi}(\tilde{\pi})-\frac{2\gamma\epsilon_{\mathrm{train}}}{(1-\gamma)^2}\mathop{\mathbb{E}}_{\substack{f\sim p_{\mathrm{train}}(\cdot)\\s\sim d^{\mu_f}(\cdot)}}\left[D_{\mathrm{TV}}(\tilde{\mu}_f\Vert\mu_f)[s]\right],
	\end{equation}
	where $\epsilon_{\mathrm{train}}=\max_{f\in\mathcal{F}_{\mathrm{train}}}\left\{\max_s\left\vert\mathbb{E}_{a\sim\tilde{\mu}_f(\cdot|s)}\left[A^{\mu_f}(s,a)\right]\right\vert\right\}$.
\end{theorem}
\begin{proof}
	See Appendix \ref{proof 1}.
\end{proof}

\begin{theorem}[Generalization performance lower bound] \label{theorem 2}
	Given any two policies, $\tilde{\pi}$ and $\pi$, the following bound holds:
	\begin{equation}
		\begin{split}
			\zeta(\tilde{\pi})&\geq L_{\pi}(\tilde{\pi})-\frac{2r_{\max}(1-Z)}{1-\gamma}-\frac{2\gamma\epsilon_{\mathrm{train}}}{(1-\gamma)^2}\mathop{\mathbb{E}}_{\substack{f\sim p_{\mathrm{train}}(\cdot)\\s\sim d^{\mu_f}(\cdot)}}\left[D_{\mathrm{TV}}(\tilde{\mu}_f\Vert\mu_f)[s]\right]\\
			&-\frac{2\delta_{\mathrm{train}}(1-Z)}{1-\gamma}\mathop{\mathbb{E}}_{\substack{f\sim p_{\mathrm{train}}(\cdot)\\s\sim d^{\tilde{\mu}_f}(\cdot)}}\left[D_{\mathrm{TV}}(\tilde{\mu}_f\Vert\mu_f)[s]\right]-\frac{2\delta_{\mathrm{eval}}(1-Z)}{1-\gamma}\mathop{\mathbb{E}}_{\substack{f\sim p_{\mathrm{eval}}(\cdot)\\s\sim d^{\tilde{\mu}_f}(\cdot)}}\left[D_{\mathrm{TV}}(\tilde{\mu}_f\Vert\mu_f)[s]\right],\\
		\end{split}
	\end{equation}
	where $r_{\max}=\max_{s,a}\left\vert r(s,a)\right\vert$, $\delta_{\mathrm{train}}=\max_{f\in\mathcal{F}_{\mathrm{train}}}\left\{\max_{s,a}\left\vert A^{\mu_f}(s,a)\right\vert\right\}$, and $\delta_{\mathrm{eval}}=\max_{f\in\mathcal{F}_{\mathrm{eval}}}\left\{\max_{s,a}\left\vert A^{\mu_f}(s,a)\right\vert\right\}$.
\end{theorem}
\begin{proof}
	See Appendix \ref{proof 2}.
\end{proof}

\textbf{Explanation.}
Building upon Theorems \ref{theorem 1} and \ref{theorem 2}, we observe that, in contrast to the lower bound on training performance, the lower bound on generalization performance incorporates three additional terms, scaled by the common coefficient $(1-Z)$. This implies that increasing $Z$ contributes to improved generalization performance, with the special case of $Z=1$ resulting in alignment between generalization and training performance. Notably, this theoretical insight was also validated in Figure 2 of \citet{cobbe2020leveraging}.

However, once the training level is fixed (i.e., $\mathcal{F}_{\mathrm{train}}$), $Z$ is a constant, improving generalization performance requires constraining the following three terms:
\begin{equation}
	\underbrace{\mathop{\mathbb{E}}_{\substack{f\sim p_{\mathrm{train}}(\cdot)\\s\sim d^{\tilde{\mu}_f}(\cdot)}}\left[D_{\mathrm{TV}}(\tilde{\mu}_f\Vert\mu_f)[s]\right]}_{\text{denote it as }\mathfrak{D}_1},\enspace\underbrace{\mathop{\mathbb{E}}_{\substack{f\sim p_{\mathrm{eval}}(\cdot)\\s\sim d^{\tilde{\mu}_f}(\cdot)}}\left[D_{\mathrm{TV}}(\tilde{\mu}_f\Vert\mu_f)[s]\right]}_{\text{denote it as }\mathfrak{D}_2},\enspace\underbrace{\mathop{\mathbb{E}}_{\substack{f\sim p_{\mathrm{train}}(\cdot)\\s\sim d^{\mu_f}(\cdot)}}\left[D_{\mathrm{TV}}(\tilde{\mu}_f\Vert\mu_f)[s]\right]}_{\text{denote it as }\mathfrak{D}_{\mathrm{train}}}.
\end{equation}
During the training process, we can only empirically bound $\mathfrak{D}_{\mathrm{train}}$. Next, we establish the upper bounds of $\mathfrak{D}_1$ and $\mathfrak{D}_2$. Specifically, we propose the following theorem:
\begin{theorem}\label{theorem 3}
	Given any two policies, $\tilde{\pi}$ and $\pi$, the following bound holds:
	\begin{equation}
		\mathfrak{D}_1\leq\left(1+\frac{2\gamma\sigma_{\mathrm{train}}}{1-\gamma}\right)\mathfrak{D}_{\mathrm{train}},\enspace\mathfrak{D}_2\leq\left(1+\frac{2\gamma\sigma_{\mathrm{eval}}}{1-\gamma}\right)\underbrace{\mathop{\mathbb{E}}_{\substack{f\sim p_{\mathrm{eval}}(\cdot)\\s\sim d^{\mu_f}(\cdot)}}\left[D_{\mathrm{TV}}(\tilde{\mu}_f\Vert\mu_f)[s]\right]}_{\text{denote it as }\mathfrak{D}_{\mathrm{eval}}},
	\end{equation}
	where $\sigma_{\mathrm{train}}=\max_{f\in\mathcal{F}_{\mathrm{train}}}\left\{D_{\mathrm{TV}}^{\max}(\tilde{\mu}_f\Vert\mu_f)[s]\right\}$ and $\sigma_{\mathrm{eval}}=\max_{f\in\mathcal{F}_{\mathrm{eval}}}\left\{D_{\mathrm{TV}}^{\max}(\tilde{\mu}_f\Vert\mu_f)[s]\right\}$, $D_{\mathrm{TV}}^{\max}(\tilde{\mu}_f\Vert\mu_f)[s]$ is defined as $\max_{s}D_{\mathrm{TV}}(\tilde{\mu}_f\Vert\mu_f)[s]$.
\end{theorem}
\begin{proof}
	See Appendix \ref{proof 3}.
\end{proof}

The only problem now is finding the relationship between $\mathfrak{D}_{\mathrm{eval}}$ and $\mathfrak{D}_{\mathrm{train}}$. To achieve this, we would like to first introduce the following definition, which represents the policy robustness to irrelevant features.
\begin{definition}[$\mathcal{R}$-robust]\label{definition 1}
	We say that the policy $\pi$ is $\mathcal{R}$-robust if it satisfies
	\begin{equation}
		\sup_{s\in\mathcal{S},\tilde{f},f\in\mathcal{F}}D_{\mathrm{TV}}(\mu_{\tilde{f}}\Vert\mu_f)[s]=\mathcal{R}.
	\end{equation}
\end{definition}

\textbf{Explanation.} This definition demonstrates how the policy $\pi$ is influenced by two different rendering functions, $\tilde{f}$ and $f$, for any given underlying state $s$. If $\mathcal{R}=0$, it indicates that $D_{\mathrm{TV}}(\mu_{\tilde{f}}\Vert\mu_f)[s]\equiv0$, which means that the policy is no longer affected by any irrelevant features. 

Our intention in this definition is not to derive the tightest possible bound but rather to demonstrate how policy robustness to irrelevant features can contribute to improved generalization. Subsequently, leveraging Definition \ref{definition 1}, we establish an upper bound for $\mathfrak{D}_{\mathrm{eval}}$.
\begin{theorem}\label{theorem 5}
	Given any two policies, $\tilde{\pi}$ and $\pi$, assume that $\tilde{\pi}$ is $\mathcal{R}_{\tilde{\pi}}$-robust, and $\pi$ is $\mathcal{R}_{\pi}$-robust, then the following bound holds:
	\begin{equation}
		\mathfrak{D}_{\mathrm{eval}}\leq\left(1+\frac{2\gamma\sigma_{\mathrm{train}}}{1-\gamma}\right)\mathcal{R}_{\pi}+\mathcal{R}_{\tilde{\pi}}+\mathfrak{D}_{\mathrm{train}}.
	\end{equation}
\end{theorem}
\begin{proof}
	See Appendix \ref{proof 5}.
\end{proof}

Altogether, by combining Theorems \ref{theorem 2}, \ref{theorem 3}, and \ref{theorem 5}, we can derive the following corollary:
\begin{corollary}
	Given any two policies, $\tilde{\pi}$ and $\pi$, the following bound holds:
	\begin{equation}
		\zeta(\tilde{\pi})\geq L_{\pi}(\tilde{\pi})-C_{\mathrm{train}}\mathfrak{D}_{\mathrm{train}}-C_{\pi}\mathcal{R}_{\pi}-C_{\tilde{\pi}}\mathcal{R}_{\tilde{\pi}}-C,
	\end{equation}
	where
	\begin{equation}
		\begin{split}
			C_{\mathrm{train}}&=\frac{2\delta_{\mathrm{train}}(1-Z)}{1-\gamma}\left(1+\frac{2\gamma\sigma_{\mathrm{train}}}{1-\gamma}\right)+\frac{2\delta_{\mathrm{eval}}(1-Z)}{1-\gamma}\left(1+\frac{2\gamma\sigma_{\mathrm{eval}}}{1-\gamma}\right)+\frac{2\gamma\epsilon_{\mathrm{train}}}{(1-\gamma)^2},\\
			C_{\pi}&=\frac{2\delta_{\mathrm{eval}}(1-Z)}{1-\gamma}\left(1+\frac{2\gamma\sigma_{\mathrm{eval}}}{1-\gamma}\right)\left(1+\frac{2\gamma\sigma_{\mathrm{train}}}{1-\gamma}\right),   \\
			C_{\tilde{\pi}}&=\frac{2\delta_{\mathrm{eval}}(1-Z)}{1-\gamma}\left(1+\frac{2\gamma\sigma_{\mathrm{eval}}}{1-\gamma}\right),\enspace C=\frac{2r_{\max}(1-Z)}{1-\gamma}.\\
		\end{split}
	\end{equation}
\end{corollary}

\begin{figure}[!t]
	\centering
	\includegraphics[scale=0.55]{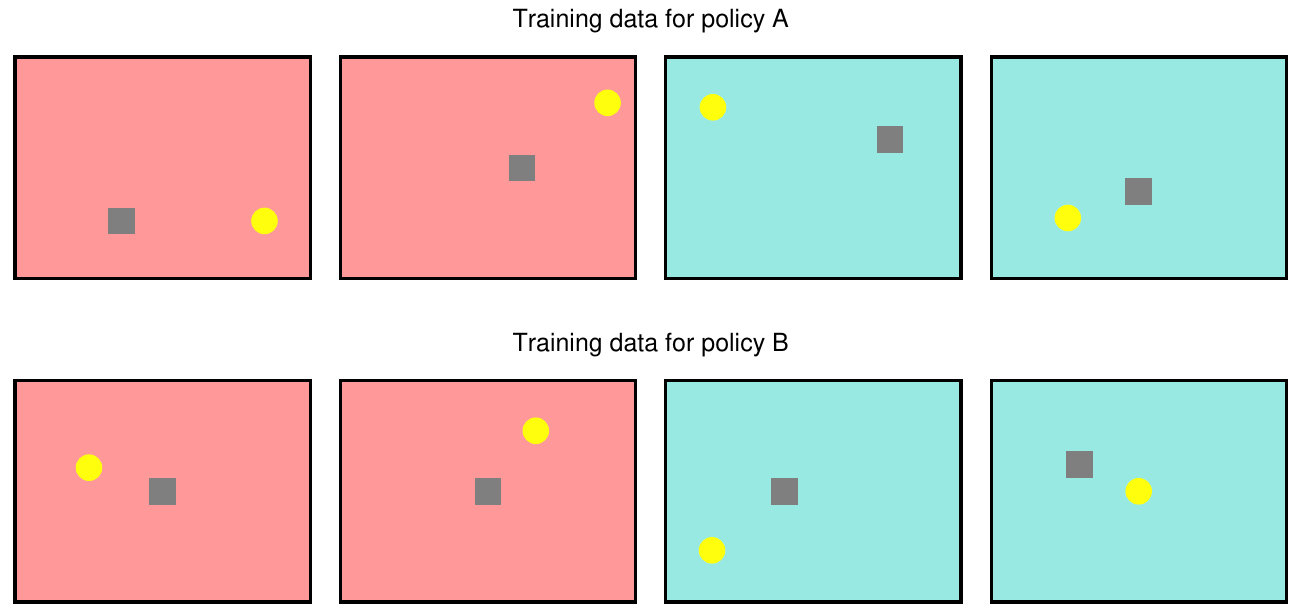}
	\caption{This is a toy environment where the gray agent's goal is to pick up coins.}
	\label{example}
\end{figure}

\textbf{Explanation.}
This represents our central theoretical result, demonstrating that enhancing generalization performance requires not only minimizing $\mathfrak{D}_{\mathrm{train}}$ during training but also improving policy robustness to irrelevant features, specifically by reducing $\mathcal{R}_{\pi}$ and $\mathcal{R}_{\tilde{\pi}}$. Furthermore, we emphasize that these results rely solely on the mild Assumptions \ref{assumption 1}, \ref{assumption 2}, and \ref{assumption 3}. Consequently, this constitutes a novel contribution that is broadly applicable to a wide range of algorithms.

\section{Distillation as Regularization}
Despite the theoretical advancements, in typical generalization settings, both the underlying MDP and the rendering function remain unknown. Next, we begin by introducing a minimal toy example in Section \ref{Toy Example}, which we then provide an in-depth analyze in Section \ref{Central Hypothesis} to motivate our hypothesis.

\subsection{Toy Example}\label{Toy Example}
Let's consider a simple environment where the agent attempts to pick up coins to earn rewards (see Figure \ref{example}). The agent's observations are the current pixels. It is clear that the agent's true objective is to pick up the coins, and the background color is a spurious feature. However, upon observing the training data for policy A, we can see that in the red background, the coins are always on the right side of the agent, while in the cyan background, the coins are always on the left side. As a result, when training policy A using reinforcement learning algorithms, it is likely to exhibit overfitting behavior, such as moving to the right in a red background and to the left in a cyan background.

However, the overfitting of policy A to the background color will fail in the training data of policy B, because in policy B's training data, regardless of whether the background color is red or cyan, the coin can appear either on the left or right side of the agent. Therefore, through DML, policy A is regularized by the behavior of policy B, effectively preventing policy A from overfitting to the background color. In other words, any irrelevant features learned by policy A could lead to suboptimal performance of policy B, and vice versa. Thus, we hypothesize that this process will force both policies to learn the true underlying semantics, ultimately improving generalization performance.

\begin{figure}[!t]
	\centering
	\includegraphics[scale=0.45]{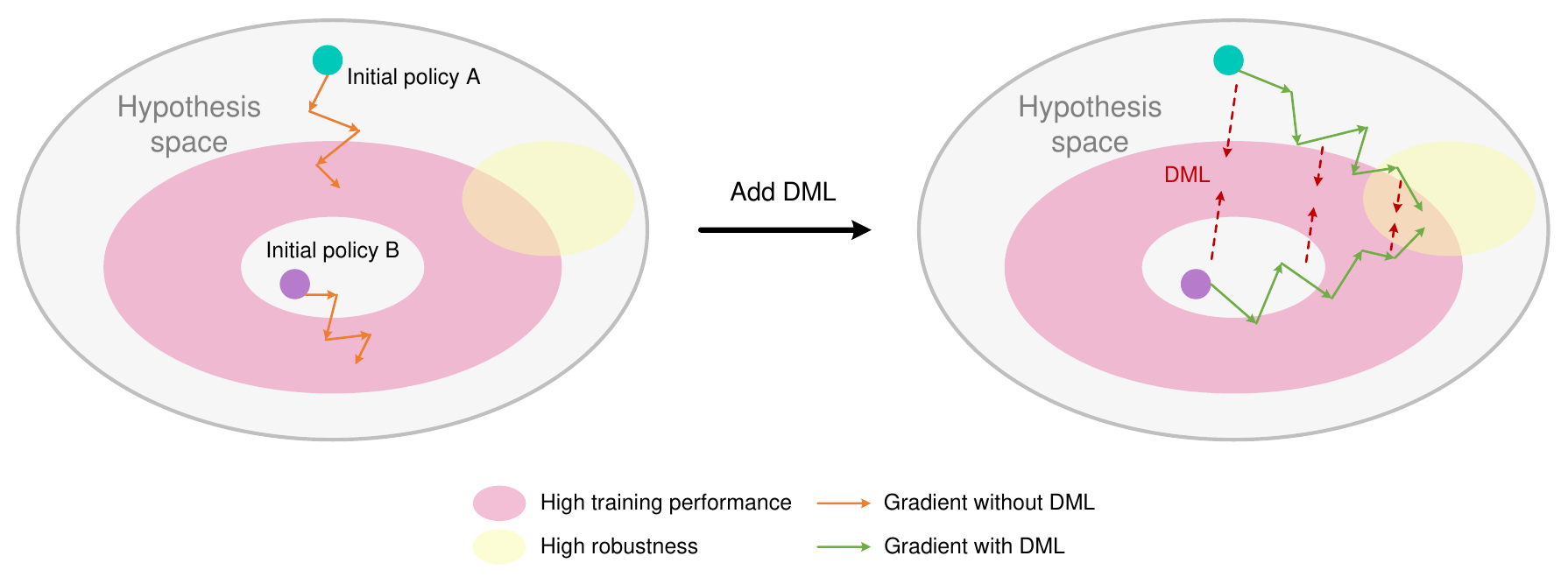}
	\caption{(Left) Independently trained reinforcement learning policies may overfit to irrelevant features. (Right) Through mutual distillation via DML, two policies regularize each other to converge toward a more robust hypothesis space, ultimately improving generalization performance.}\label{RL+DML}
\end{figure}

\subsection{Hypothesis}\label{Central Hypothesis}
Motivated by Section \ref{Toy Example}, DML can be viewed as a form of implicit regularization against irrelevant features, as demonstrated in Figure \ref{RL+DML}, which illustrates two randomly initialized policies independently trained using reinforcement learning algorithms. In this case, since the training samples only include a portion of all possible MDPs, the policies are likely to overfit to irrelevant features and fail to converge to a robust hypothesis space.

Applying DML to the training process of both policies facilitates mutual learning, which can mitigate overfitting to irrelevant features. Due to the randomness of parameter initialization and the interaction process, they generate different training samples, DML encourages both policies to make consistent decisions based on the same observations. As discussed in Section \ref{Toy Example}, any irrelevant features learned by policy A are likely to degrade the performance of policy B, and vice versa. As training progresses, DML will drive both policies to learn more meaningful and useful representations, gradually reducing the divergence between them. Ideally, we hypothesize that both policies will capture the essential aspects of high-dimensional observations as time grows.

\section{Experiments}
This section presents our main empirical results. Section \ref{Implementation Details} introduces the implementation details, Section \ref{Empirical Results} validates the effectiveness of DML technique for improving generalization performance, Section \ref{Robustness Testing} verifies our central hypothesis, and Section \ref{Ablation Study} confirms our theoretical results.

\subsection{Implementation Details}\label{Implementation Details}
We use Procgen \citep{cobbe2019quantifying, cobbe2020leveraging} as the experimental benchmark for testing generalization performance. Procgen is a suite of 16 procedurally generated game-like environments designed to benchmark both sample efficiency and generalization in reinforcement learning, and it has been widely used to test the generalization performance of various reinforcement learning algorithms \citep{wang2020improving, raileanu2021decoupling, raileanu2021automatic, lyle2022learning, rahman2023adversarial, jesson2024improving}.

We employ the Proximal Policy Optimization ({\color{myblue}PPO}) \citep{schulman2017proximal, cobbe2020leveraging} as our baseline. Specifically, given a parameterized policy $\pi_{\theta}$ ($\theta$ represents the parameters), the objective of $\pi_{\theta}$ is to maximize
\begin{equation}
	J(\theta)=\mathop{\mathbb{E}}_{(o_t,a_t)\sim\pi_{\theta_{\mathrm{old}}}}\left\{\min\left[r_t(\theta)\cdot\hat{A}(o_t,a_t),\mathrm{clip}\left(r_t(\theta),1-\epsilon,1+\epsilon\right)\cdot\hat{A}(o_t,a_t)\right]\right\},
\end{equation}
where $\hat{A}$ is the advantage estimate, and $r_t(\theta) = \pi_{\theta}(a_t|o_t) / \pi_{\theta_{\mathrm{old}}}(a_t|o_t)$ is the probability ratio, where $\pi_{\theta_{\mathrm{old}}}$ and $\pi_{\theta}$ denote the old and current policies, respectively.

We randomly initialize two agents to interact with the environment and collect data separately. Similar to the DML loss (\ref{dml loss}) used in supervised learning, we also introduce an additional KL divergence loss term, which leads to
\begin{equation}
	\mathcal{L}_{\mathrm{DML}}=\mathcal{L}_{\mathrm{RL}}+\alpha \mathcal{L}_{\mathrm{KL}},
\end{equation}
where $\mathcal{L}_{\mathrm{RL}}$ is the reinforcement learning loss and $\mathcal{L}_{\mathrm{KL}}$ is the KL divergence loss, $\alpha$ is the weight. And then we optimize the total loss of both agents, which is the average of their DML losses, as shown in Algorithm \ref{PPO with DML}, which we name Mutual Distillation Policy Optimization ({\color{myred}MDPO}).

\begin{algorithm}[!h]
	\caption{Mutual Distillation Policy Optimization ({\color{myred}MDPO})}
	\label{PPO with DML}
	\begin{algorithmic}[1]
		\STATE {\bfseries Initialize:} Two agents $\pi_1,\pi_2$, PPO algorithm $\mathcal{A}$, KL divergence weight $\alpha$
		\WHILE{training}
		\FOR{$i=1,2$}
		\STATE Collect training data: $\mathcal{D}_i\sim\pi_i$
		\STATE Compute RL loss: $\mathcal{L}_{\mathrm{RL}}^{(i)}\leftarrow\mathcal{A}(\mathcal{D}_i)$
		\STATE Compute KL loss: $\mathcal{L}_{\mathrm{KL}}^{(i)}\leftarrow D_{\mathrm{KL}}(\pi_{3-i}\Vert\pi_i)$
		\STATE\color{myred} Compute DML loss: $\mathcal{L}_{\mathrm{DML}}^{(i)}\leftarrow\mathcal{L}_{\mathrm{RL}}^{(i)}+\alpha \mathcal{L}_{\mathrm{KL}}^{(i)}$
		\ENDFOR
		\STATE Compute total loss: $\mathcal{L}\leftarrow\frac{1}{2}\left(\mathcal{L}_{\mathrm{DML}}^{(1)}+\mathcal{L}_{\mathrm{DML}}^{(2)}\right)$
		\STATE Optimize $\mathcal{L}$ using gradient descent algorithm
		\ENDWHILE
	\end{algorithmic}
\end{algorithm}

Ultimately, we do not claim to achieve state-of-the-art (SOTA) performance, but rather provide empirical evidence for the non-trivial insight that DML serves as an implicit regularization against irrelevant features, leading to consistent improvements in generalization performance. We also acknowledge the methodological similarities with prior work such as \citet{zhao2021robust}; despite that, we introduce \textit{representation convergence} (Section \ref{Central Hypothesis}), a novel insight with further supported by strong theoretical analysis (Section \ref{Theoretical Results}), constituting our additional contributions.

\subsection{Empirical Results}\label{Empirical Results}
We compare the generalization performance of our {\color{myred}MDPO} against the {\color{myblue}PPO} baseline on the Procgen benchmark, under the hard-level settings \citep{cobbe2020leveraging}, the results are illustrated in Figure \ref{generalization performance}. It can be observed that DML technique indeed leads to consistent improvements in generalization performance across all environments. Notably, for the bigfish, dodgeball, and fruitbot environments, we have observed significant improvements. Moreover, the full experimental results for all environments, including training and generalization performance, are provided in Appendix \ref{More Experimental Results}.

A natural concern arises: how can we determine whether DML improves generalization performance by enhancing the policy robustness against irrelevant features, or simply due to the additional information sharing between these two agents during training (each agent receives additional information than it would from training alone)? To answer this question, we conducted robustness testing in Section \ref{Robustness Testing} and added an ablation study in Section \ref{Ablation Study} to support our theory and hypothesis.

\begin{figure}[!t]
	\centering
	\includegraphics[width=\textwidth]{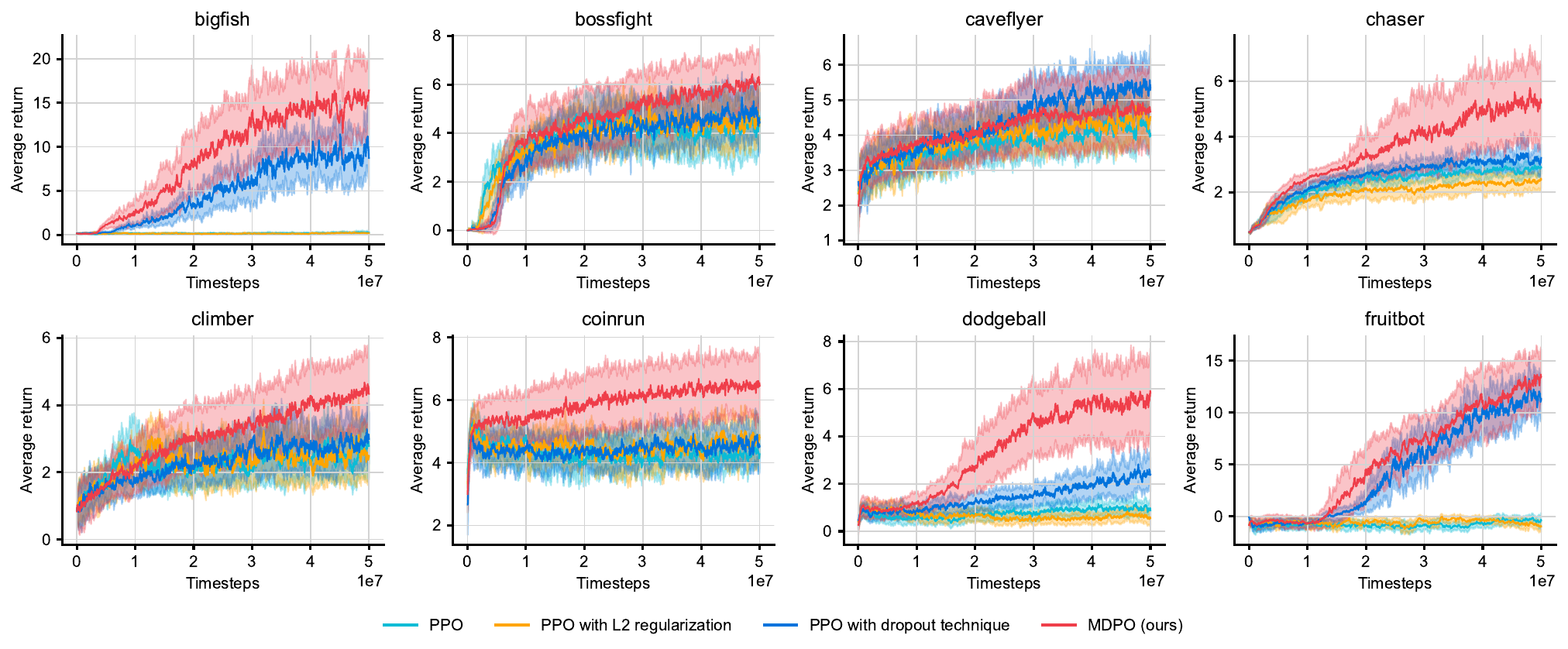}
	\caption{Generalization performance from 500 levels in Procgen benchmark with different methods. Our {\color{myred}MDPO} gains significant performance improvement compared with the baseline algorithms.}\label{generalization performance}
\end{figure}

\begin{figure}[!t]
	\centering
	\includegraphics[scale=0.38]{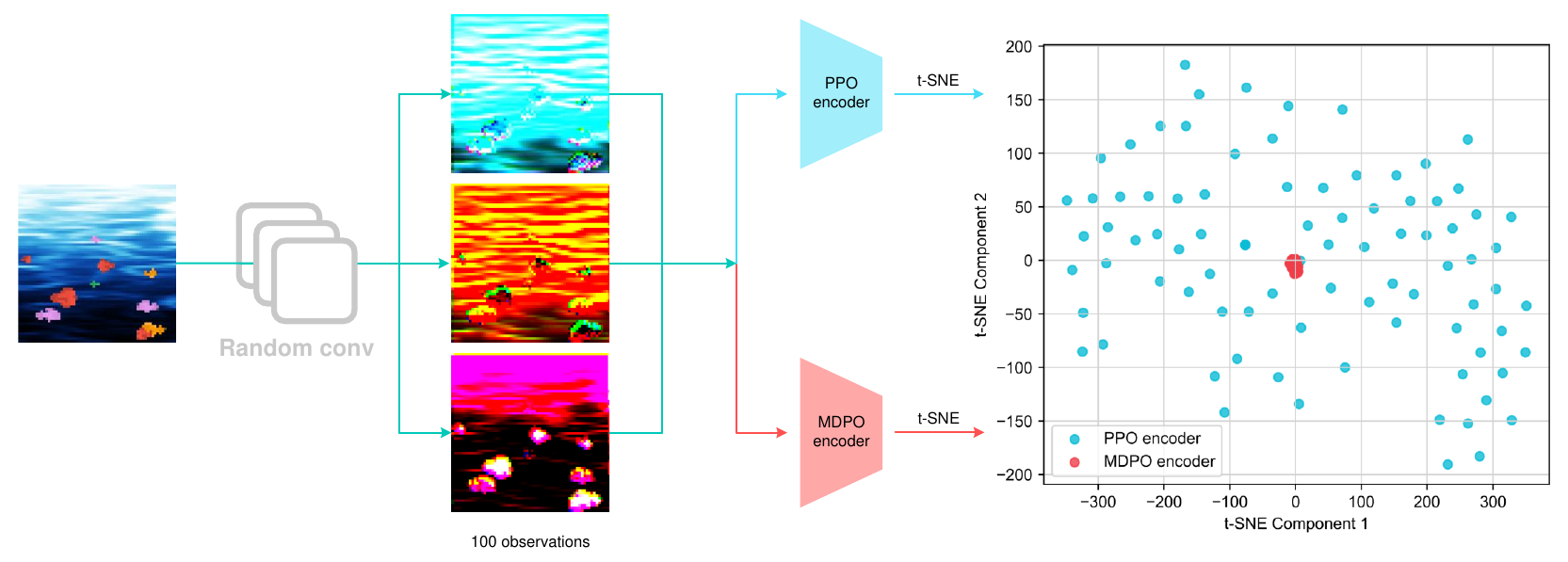}
	\caption{To test the robustness of the trained policy, we obfuscate the agent's observations using convolutional layers randomly initialized with a standard Gaussian distribution.}
	\label{random conv}
\end{figure}

\begin{wrapfigure}{r}{0.4\textwidth}
	\centering
	\tiny
	\begin{tabular}{c|cccc}
		\toprule
		Algo\textbackslash Env & bigfish & bossfight & caveflyer & chaser \\
		\midrule
		{\color{myblue}PPO} & 10.65 & 10.88 & 10.35 & 16.51\\
		{\color{myred}MDPO} & \bf0.99 & \bf0.94 & \bf0.97 & \bf1.00\\
		\bottomrule
		\toprule
		Algo\textbackslash Env & fruitbot & jumper & miner & ninja \\
		\midrule
		{\color{myblue}PPO} & 9.93 & 14.88 & 18.32 & 9.73\\
		{\color{myred}MDPO} & \bf0.99 & \bf0.99 & \bf0.98 & \bf1.00\\
		\bottomrule
	\end{tabular}
	\captionof{table}{A simple practical measure of $\mathcal{R}$-robustness defined in Definition \ref{definition 1}.}\label{robust}
\end{wrapfigure}

\subsection{Robustness Testing}\label{Robustness Testing}
We design a novel approach to test policy robustness against irrelevant features. For a given frame, we generate \textit{adversarial samples} using random CNNs initialized with a standard Gaussian distribution, as shown in Figure \ref{random conv}. Notably, the feature extraction of {\color{myred}MDPO} encoder is highly stable and focused ({\color{myred}red points}), whereas the features extracted by the original {\color{myblue}PPO} encoder are significantly dispersed ({\color{myblue}blue points}).

Moreover, we design a practical measure of $\mathcal{R}$-robustness defined in Definition \ref{definition 1}. Specifically, for each environment, we run the trained policy ({\color{myblue}PPO} and {\color{myred}MDPO}) in the environment for 100 steps and obtain observations $o_1,o_2,\dots,o_{100}$. Then, for each $o_i$ we use 100 random CNNs to simulate rendering function samples $f_1^{(i)},f_2^{(i)},\dots,f_{100}^{(i)}$ and compute the TV divergence of the policy between the adversarial samples and the original observations, i.e., $D_{\mathrm{TV}}(\pi_{\theta}(\cdot|o_i)\Vert\pi_{\theta}(\cdot|f_{j}^{(i)}(o_i)))$, where $i,j=1,2,\dots,100$. We then take the maximum of these values as a simple practical measure of $\mathcal{R}$-robustness:
\begin{equation}
	\hat{\mathcal{R}}:=\max_{i,j}D_{\mathrm{TV}}(\pi_{\theta}(\cdot|o_i)\Vert\pi_{\theta}(\cdot|f_{j}^{(i)}(o_i))),
\end{equation}
the results are shown in Table \ref{robust}. We can see that {\color{myred}MDPO} achieves a significantly lower $\hat{\mathcal{R}}$ than {\color{myblue}PPO}, showing that DML effectively improves the policy robustness to irrelevant features, which serves as further strong evidence for our hypothesis.

\subsection{Ablation Study}\label{Ablation Study}
\begin{wrapfigure}{r}{0.55\textwidth}
	\centering
	\tiny
	\begin{tabular}{c|cccc}
		\toprule
		Algo\textbackslash Env & bigfish & chaser & dodgeball & fruitbot \\
		\midrule
		PPO ({\color{myblue}PPO} encoder) & 0.19$^{\pm\text{0.14}}$ & 2.57$^{\pm\text{0.28}}$ & 0.71$^{\pm\text{0.34}}$ & -0.39$^{\pm\text{0.46}}$\\
		PPO ({\color{myred}MDPO} encoder) & {\bf22.67}$^{\pm\text{6.40}}$ & {\bf6.22}$^{\pm\text{1.36}}$ & {\bf4.70}$^{\pm\text{1.91}}$ & {\bf11.22}$^{\pm\text{2.16}}$\\
		\bottomrule
	\end{tabular}
	\captionof{table}{Generalization performance of PPO linear probe on top of the \textit{frozen} encoders.}\label{frozen}
	\begin{tabular}{c|cccc}
		\toprule
		$\alpha$ in {\color{myred}MDPO}\textbackslash Env & bigfish & chaser & dodgeball & fruitbot \\
		\midrule
		0 (baseline) & 0.26$^{\pm\text{0.23}}$ & 0.92$^{\pm\text{0.46}}$ & -0.50$^{\pm\text{0.81}}$ &3.99$^{\pm\text{0.21}}$ \\
		0.1 & 9.87$^{\pm\text{4.57}}$ & 4.35$^{\pm\text{1.63}}$ & 11.94$^{\pm\text{2.96}}$ &10.97$^{\pm\text{2.72}}$ \\
		1 & {\bf16.11}$^{\pm\text{4.63}}$ & {\bf5.66}$^{\pm\text{1.98}}$ & {\bf13.23}$^{\pm\text{3.04}}$ &{\bf11.28}$^{\pm\text{3.04}}$ \\
		10 & {7.69}$^{\pm\text{3.65}}$ & {4.35}$^{\pm\text{1.48}}$ & {2.31}$^{\pm\text{2.41}}$ &{8.54}$^{\pm\text{2.27}}$ \\
		\bottomrule
	\end{tabular}
	\captionof{table}{Generalization performance of MDPO under different KL divergence weights.}\label{sensitivity analysis}
\end{wrapfigure}
We design additional ablation experiments. Specifically, we \textit{double} the model size, batch size, and total number of interactions for the PPO baseline, as shown in Figure \ref{double}. It can be seen that PPO baseline still fails to match the performance of {\color{myred}MDPO}, demonstrating that naively scaling up the {\color{myblue}PPO} baseline does not lead to stable improvements in generalization performance. 

Furthermore, we retrain a PPO \textit{linear probe} on top of the \textit{frozen} encoders of the trained {\color{myblue}PPO} and {\color{myred}MDPO} policies, training for only 1M steps (2\% of the original training steps), the final generalization performance during the last 10\% steps is shown in Table \ref{frozen}. It can be seen that the PPO linear probe trained on the {\color{myred}MDPO} encoder achieves significantly better generalization performance, indicating that DML helps the policy learn better (more robust) representations. Moreover, we add a sensitivity analysis of the KL divergence weight $\alpha$, and the results are presented in Table \ref{sensitivity analysis}.

\begin{figure}[!t]
	\centering
	\includegraphics[width=\textwidth]{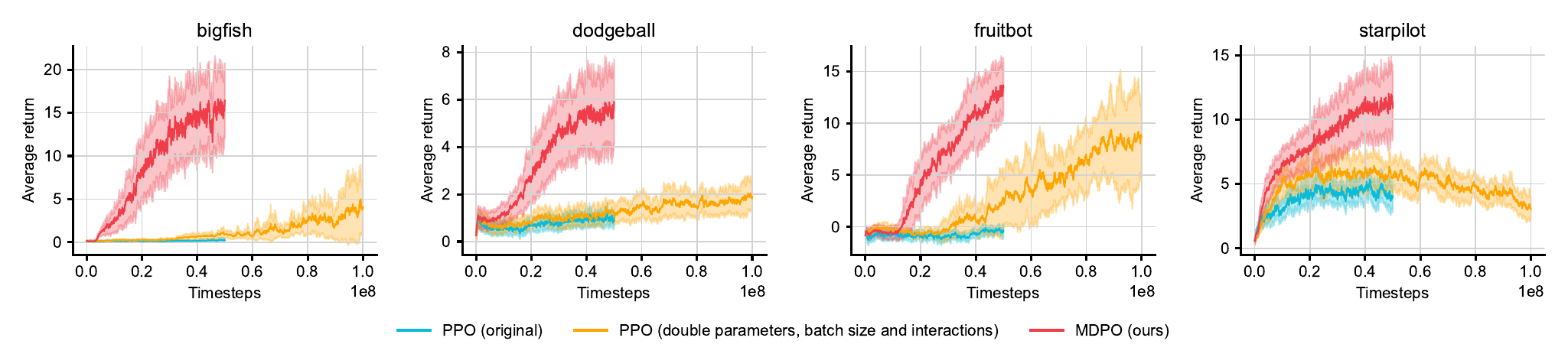}
	\caption{Generalization performance of {\color{orange}PPO baseline with double model size, batch size, and total number of interactions}, compared to original {\color{myblue}PPO} and {\color{myred}MDPO} (for training results, see Figure \ref{double train}).}\label{double}
\end{figure}

\section{Conclusion}
In this paper, we provide a novel theoretical framework to explain the generalization problem in deep reinforcement learning. We further hypothesize that DML, as a form of implicit regularization, effectively prevents the policy from overfitting to irrelevant features. Strong empirical results support our central theory and hypothesis, demonstrating that our approach can improve the generalization performance of reinforcement learning systems by enhancing robustness against irrelevant features. Our work provides valuable insights and elegant solutions into the development of more adaptable and robust policies capable of generalizing across diverse environments.


%

\bibliography{iclr2026_conference}
\bibliographystyle{iclr2026_conference}

\appendix
\newpage

\section{Related Work}
\textbf{The generalization of deep reinforcement learning} has been widely studied, and previous work has pointed out the overfitting problem in deep reinforcement learning \citep{rajeswaran2017towards, zhang2018study, justesen2018illuminating, packer2018assessing, song2019observational, cobbe2019quantifying, grigsby2020measuring, cobbe2020leveraging, yuan2023rl}. A natural approach to avoid the overfitting problem in deep reinforcement learning is to apply regularization techniques originally developed for supervised learning such as dropout \citep{srivastava2014dropout, farebrother2018generalization, igl2019generalization}, data augmentation \citep{laskin2020reinforcement, yarats2021image, zhang2021generalization, raileanu2021automatic, ma2022comprehensive}, domain randomization \citep{tobin2017domain, yue2019domain, slaoui2019robust, mehta2020active}, or network randomization technique \citep{lee2019network}. On the other hand, in order to improve sample efficiency, previous studies encouraged the policy network and value network to share parameters \citep{schulman2017proximal, huang2022cleanrl}. However, recent works have explored the idea of decoupling the two and proposed additional distillation strategies \citep{cobbe2021phasic, raileanu2021decoupling, moon2022rethinking}. In particular, \citet{raileanu2021decoupling} demonstrated that more information is needed to accurately estimate the value function, which can lead to overfitting.

\textbf{Knowledge distillation} is a learning paradigm that aims to align the student network with the teacher network to achieve knowledge transfer. A commonly used practice is to distill the knowledge learned by a large model into a smaller model to reduce inference costs after deployment \citep{xu2024survey}. On the other hand, distillation technique can also be used to distill a model with privileged information into a model with access to only partial information to improve its generalization ability. However, research has shown that knowledge distillation can also be applied to multiple student networks during training to encourage them to learn from each other, called deep mutual learning (DML) \citep{zhang2018deep}. Building upon this observation, \citet{zhao2021robust} further demonstrate that DML can improve the generalization performance of reinforcement learning agents, yet no in-depth analysis of why this happens. In addition, recent studies suggest that aligning the student networks at the output layer may be suboptimal, and recommend alignment at the logits layer instead \citep{deckers2024twin, vandersmissenimproving}.

\section{Hyperparameters} \label{Hyperparameters}
Table \ref{hyperparameters} shows the detailed hyperparameter settings in our code, with the main hyperparameters consistent with the hard-level settings in \citet{cobbe2020leveraging}, except that we train for 50M steps instead of 200M. We train the policy on the initial 500 levels and then test its generalization performance across the full distribution of levels.
\begin{table}[!h]
	\scriptsize
	\centering
	\caption{Detailed hyperparameters in Procgen.}
	\label{hyperparameters}
	\begin{tabular}{c|cc}
		\toprule
		Hyperparameter\textbackslash Algorithm & PPO \citep{schulman2017proximal} & MDPO (ours) \\
		\midrule
		Number of workers & 64 & 64  \\
		Horizon & 256& 256 \\
		Learning rate& 0.0005 & 0.0005  \\
		Learning rate decay & No& No \\
		Optimizer & Adam & Adam  \\
		Total interaction steps & 50M & 50M  \\
		Update epochs &3  &3 \\
		Mini-batches & 8  & 8 \\
		Batch size & 16384 & 16384  \\
		Mini-batch size &  2048 &  2048 \\
		Discount factor $\gamma$  & 0.999 & 0.999\\
		GAE parameter $\lambda$  & 0.95  & 0.95\\
		Value loss coefficient $c_1$ & 0.5 & 0.5 \\
		Entropy loss coefficient $c_2$ & 0.01 & 0.01 \\
		Clipping parameter $\epsilon$ & 0.2 & 0.2 \\
		KL divergence weight $\alpha$ & - & 1.0 \\
		\bottomrule
	\end{tabular}
\end{table}

\newpage
\section{Philosophical Insight}
An intriguing analogy for our hypothesis is the process of truth emergence. Typically, each scholar offers their unique perspective, but for it to be widely accepted, it must garner consensus from peers within the field, or even from the broader academic community. We can draw a parallel between DML and the peer review process: when a particular viewpoint is accepted by the majority, it is more likely to reflect an objective truth. Going deeper, our hypothesis aligns with the philosophical concept of \textit{convergent realism} \citep{laudan1981confutation, kelly1989convergence, huh2024platonic}, which posits that science progresses towards an objective truth. 

In the \textit{Allegory of the Cave} \citep{cohen2006the}, Plato describes a group of people who have been chained in a cave from birth, facing the wall. They can only see the shadows cast on the wall by objects behind them, illuminated by a fire. These shadows are all they know and perceive as reality. One day, a prisoner is freed and exits the cave. Upon seeing the world outside, he discovers that the true reality is entirely different from the shadows, and he realizes that what they thought was ``real'' was only an illusion.

The Allegory of the Cave explores the concept of different levels of knowledge. From a metaphysical standpoint, Plato is distinguishing between the ``appearance'' and the ``reality''. The shadows in the cave represent the sensory world, which is just an illusion, while the outside world symbolizes the realm of true reality, which can only be apprehended through rational thought and philosophical inquiry. Plato argues that what we perceive with our senses is not the true essence of reality, what he called the \textit{ideal reality}, but rather a mere shadow of the higher, eternal truths.

However, even though the freed prisoner has seen the world outside the cave, how can he be certain that what he now perceives is not a shadow of something even more fundamental? This is exactly \textit{skepticism}, the philosophical stance that questions the possibility of certain knowledge. It suggests that every layer of perceived reality might itself be an illusion, prompting thinkers from Descartes onward to doubt not only the senses but even the existence of the external world, demanding ever more rigorous foundations for truth. To counter skepticism, some philosophers have found it necessary to expand the definition of ``reality'' and propose the \textit{virtual realism} \citep{chalmers2022reality+}.

\begin{wrapfigure}{r}{0cm}
	\centering
	\includegraphics[scale=0.25]{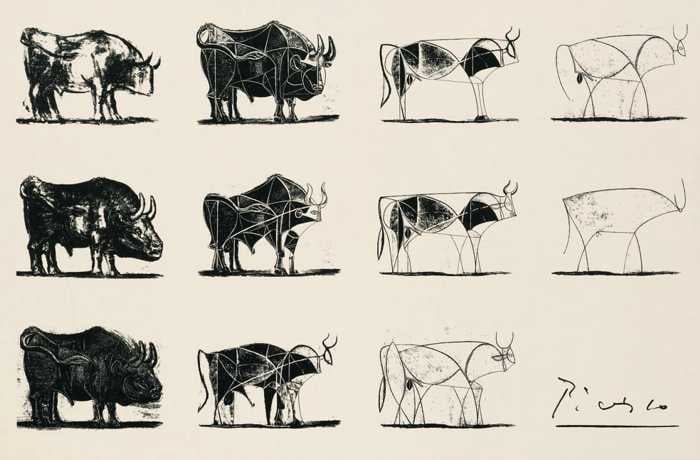}
	\caption{Picasso's \textit{The Bull} \citep{drawpaintacademy}. By focusing on and exaggerating specific details, rather than trying to capture every detail realistically, artists can convey the core meaning or essence of the subject, more powerfully.}\label{abstraction}
\end{wrapfigure}

This paper offers another powerful perspective against skepticism: if we regard the \textit{underlying state} $s_t$ as the essence of things, and the \textit{rendering function} $f$ as the projection (possibly a composition of multiple layers of projections), then our perception of the world becomes an observation $o_t = f(s_t)$. By encouraging agents to make consistent decisions based on the same observations, our method fosters a process akin to \textit{cognitive alignment} \citep{falandays2022emergence}, which has been fundamental in human societal development. For instance, in voting, the majority rule is employed because decisions supported by the majority are perceived as more reliable. Similarly, our method facilitates cognitive alignment between agents, enabling them to converge on objective truths despite noisy or irrelevant features in their observations. Over time, the cognitive alignment between agents encourages the convergence of their individual representations toward a more accurate understanding of the environment.

It is worth noting that metaphysical ideas can also be observed in certain works of art, such as Picasso's \textit{The Bull} (Figure \ref{abstraction}). Most people think that art is all about seeing more detail, but it is really about seeing less. Seeing basic patterns amongst the ``noise''; seeing basic forms amongst the complex; seeing the few important details which convey the majority of meaning.

\newpage
\section{More Results}\label{More Experimental Results}
\subsection{Full Results}
\begin{figure}[!h]
	\centering
	\includegraphics[scale=0.34]{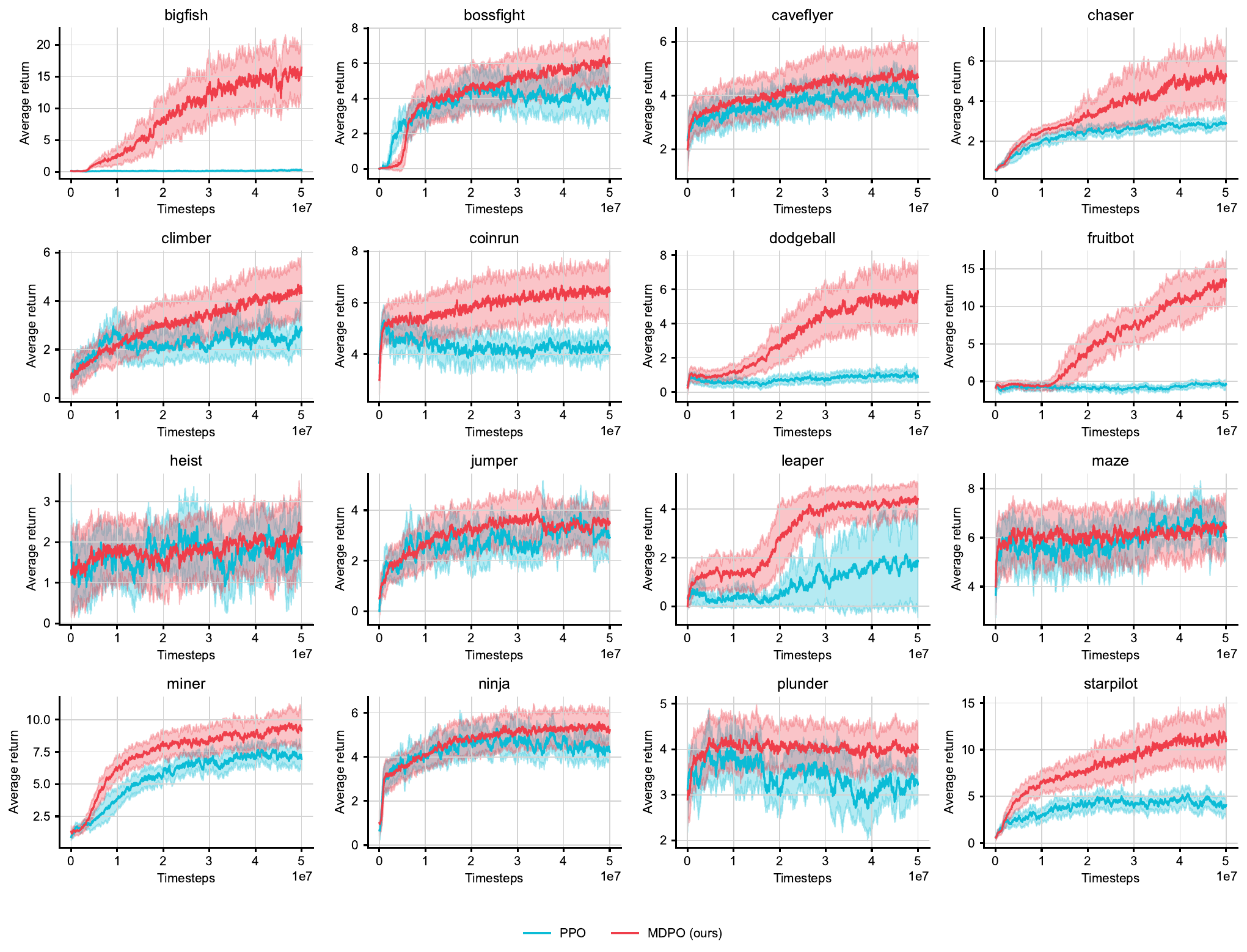}
	\caption{Generalization performance of {\color{myblue}PPO} and {\color{myred}MDPO} from 500 levels in each environment.}\label{test}
\end{figure}
\begin{figure}[!h]
	\centering
	\includegraphics[scale=0.34]{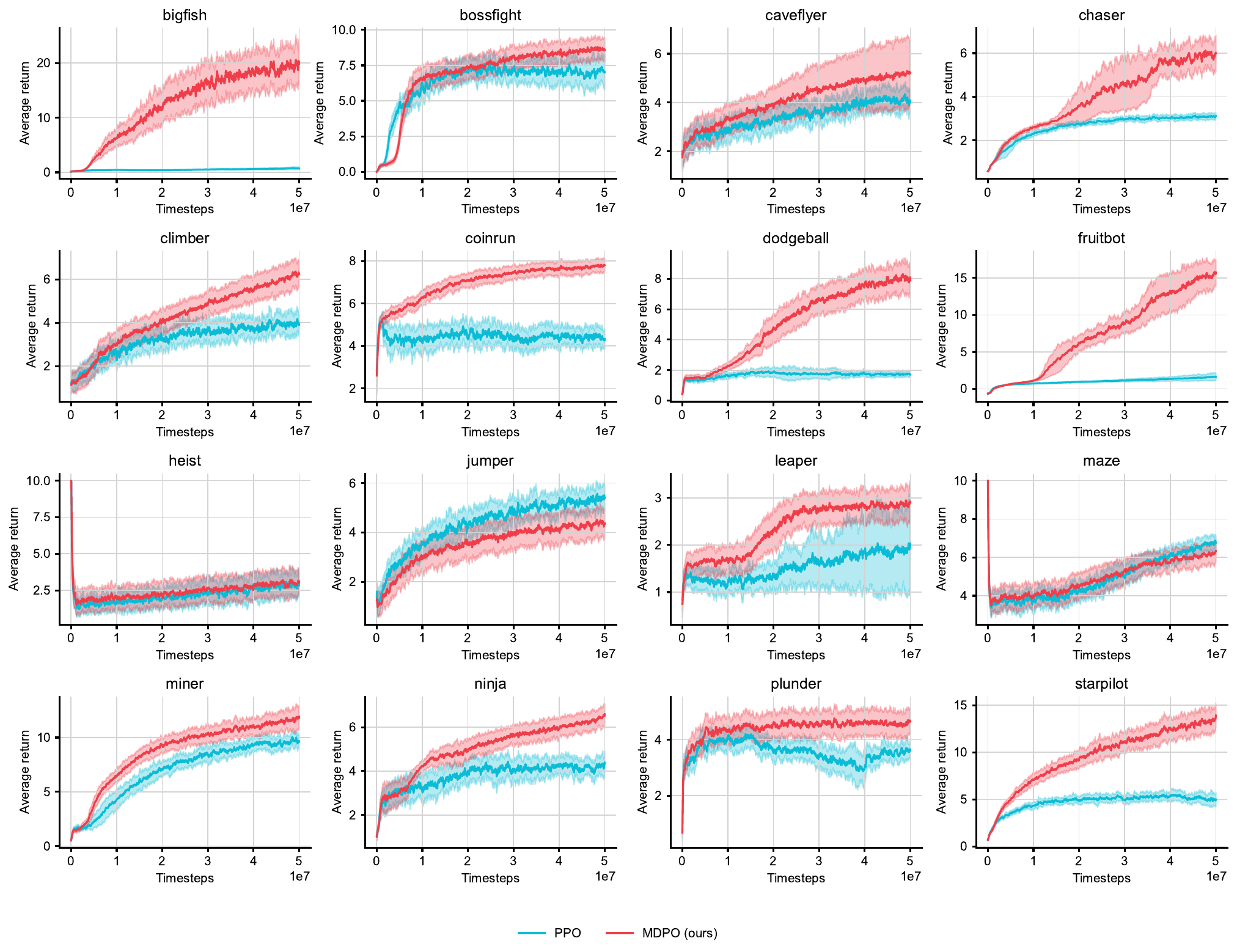}
	\caption{Training performance of {\color{myblue}PPO} and {\color{myred}MDPO} from 500 levels in each environment.}\label{train}
\end{figure}

\subsection{Other Baseline}
We select Simple Policy Optimization (SPO) \citep{xie2024simple} as another baseline and compare the generalization performance of SPO with DML-based SPO, the results are shown in Table \ref{spo test}.
\begin{table}[!h]
	\footnotesize
	\centering
	\caption{generalization performance of SPO and SPO with DML.}
	\label{spo test}
	\begin{tabular}{c|cccc}
			\toprule
			Algo\textbackslash Env & bigfish & dodgeball & fruitbot & starpilot\\
			\midrule
			SPO & {1.04}$^{\pm\text{0.96}}$ & {1.59}$^{\pm\text{0.87}}$ & {\bf2.28}$^{\pm\text{1.64}}$ &{6.06}$^{\pm\text{1.54}}$\\
			SPO with DML & {\bf4.62}$^{\pm\text{2.60}}$ & {\bf4.66}$^{\pm\text{1.44}}$ & {1.13}$^{\pm\text{1.15}}$ &{\bf8.65}$^{\pm\text{2.47}}$\\
			\bottomrule
		\end{tabular}
\end{table}

We can see that DML can also significantly improve the generalization performance of SPO, which fully demonstrates that DML-based policy optimization is a general learning framework for enhancing the generalization ability of diverse reinforcement learning algorithms.

\subsection{More Ablation Results}
Here, we additionally present the training curves from the Ablation Study (Section \ref{Ablation Study}), as shown in Figure \ref{double train}.
\begin{figure}[!h]
	\centering
	\includegraphics[width=\textwidth]{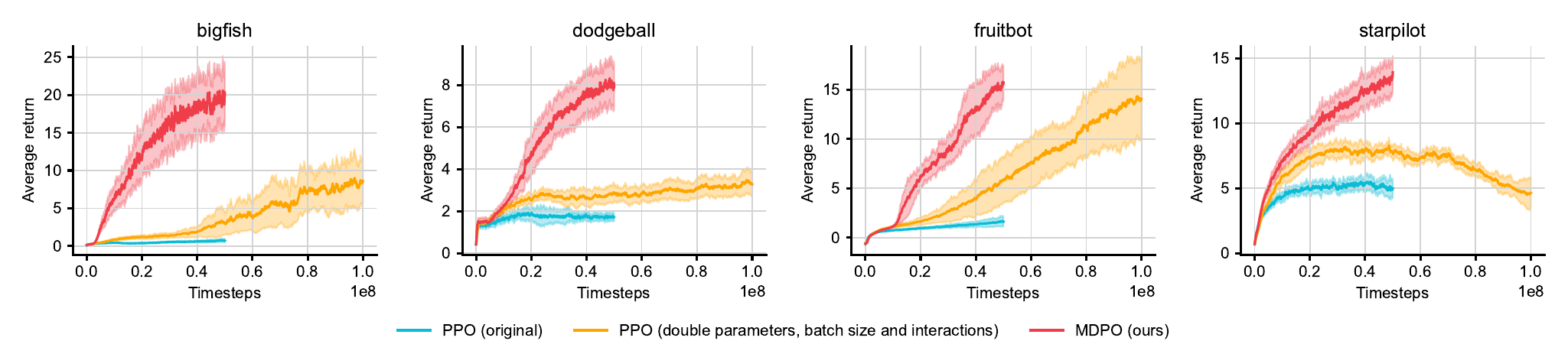}
	\caption{Training performance of {\color{orange}PPO baseline with double model size, batch size, and total number of interactions}, compared to original {\color{myblue}PPO} and {\color{myred}MDPO}.}\label{double train}
\end{figure}

Interestingly, although the {\color{orange}scaled-up PPO} nearly matches {\color{myred}MDPO} in training performance during the final stage of training in the fruitbot environment, there remains a substantial gap in their generalization performance (as shown in Figure \ref{double}). This provides further strong evidence that DML effectively enhances the policy robustness to irrelevant features, as {\color{myred}MDPO} achieves significantly better generalization performance despite comparable training performance.

\newpage
\subsection{Additional Visualizations}
We also generate adversarial samples by adjusting the brightness, contrast, saturation, and hue of the images, and test the robustness of the {\color{myblue}PPO} encoder and our {\color{myred}MDPO} encoder, as shown in Figure \ref{t-SNE-new}.
\begin{figure}[!h]
	\centering
	\includegraphics[scale=0.4]{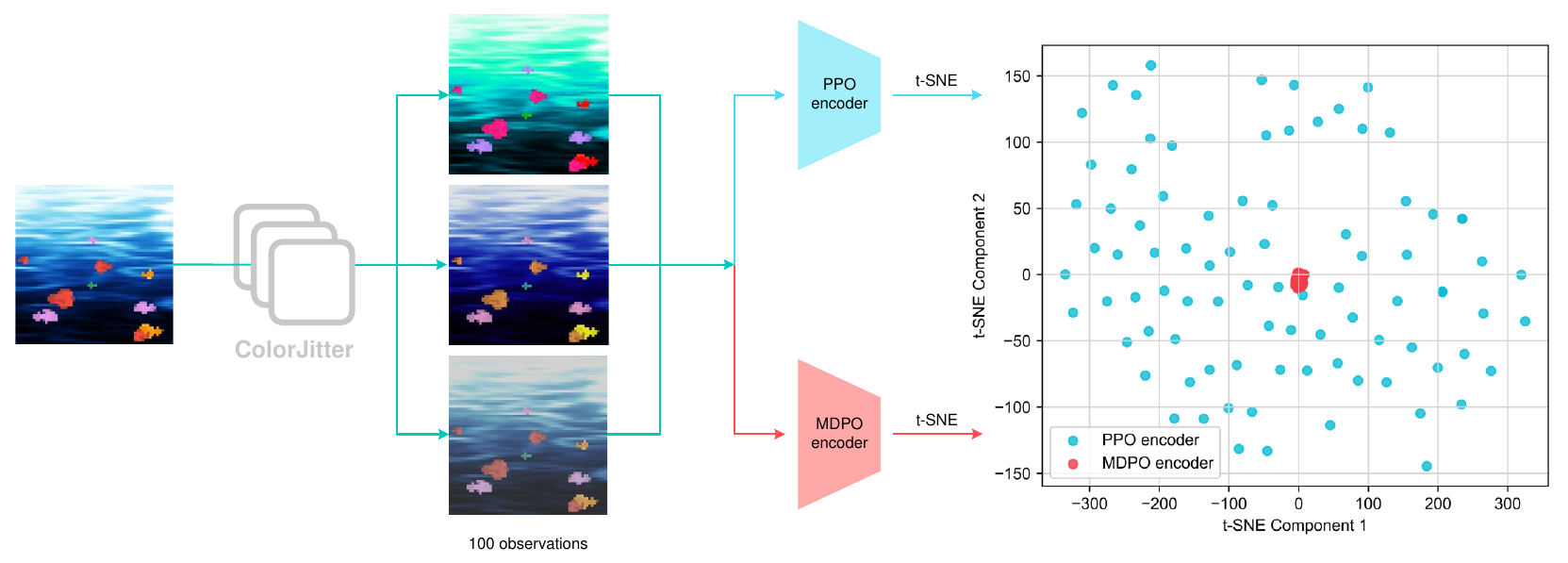}
	\caption{The robustness of PPO and MDPO to brightness, contrast, saturation, and hue.}
	\label{t-SNE-new}
\end{figure}

We can see that the {\color{myred}MDPO} policy has also learned robustness representations to these irrelevant factors, while the {\color{myblue}PPO} policy remains sensitive to them. Additionally, we present adversarial samples generated by random CNNs, as shown in Figure \ref{adversarial samples 1}, as well as those generated by randomly adjusting brightness, contrast, saturation, and hue, as can be seen from Figure \ref{adversarial samples 2}.
\begin{figure}[!h]
	\centering
	\includegraphics[scale=0.17]{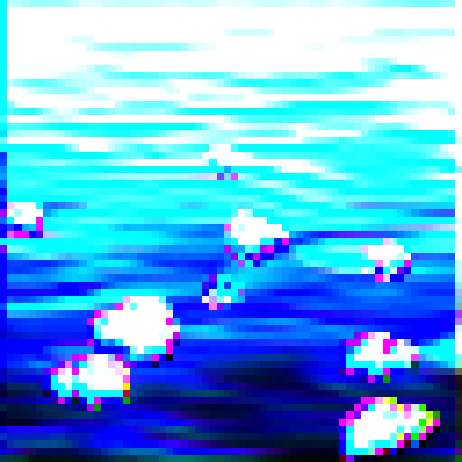}\includegraphics[scale=0.17]{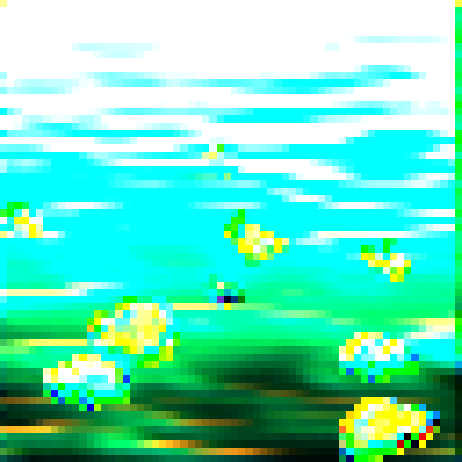}\includegraphics[scale=0.17]{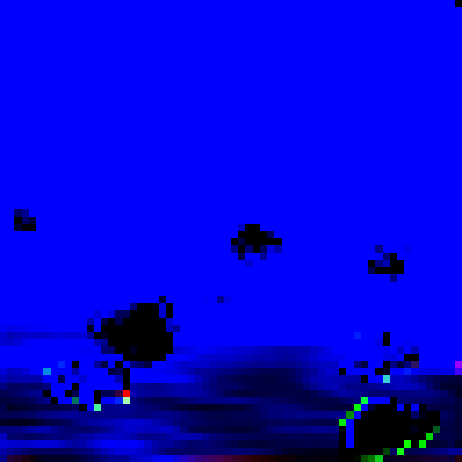}\includegraphics[scale=0.17]{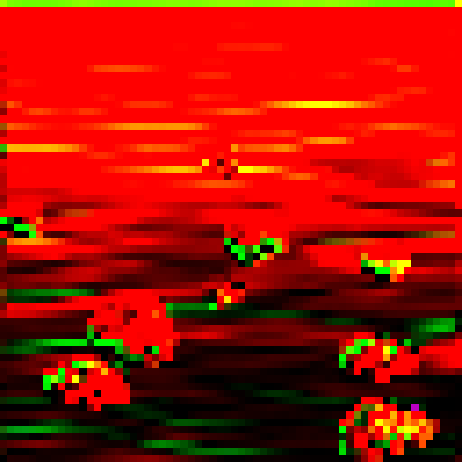}\includegraphics[scale=0.17]{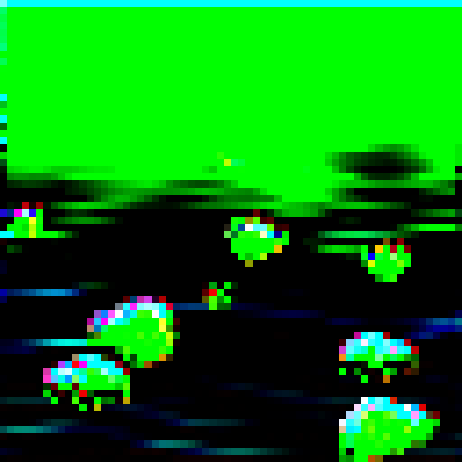}\includegraphics[scale=0.17]{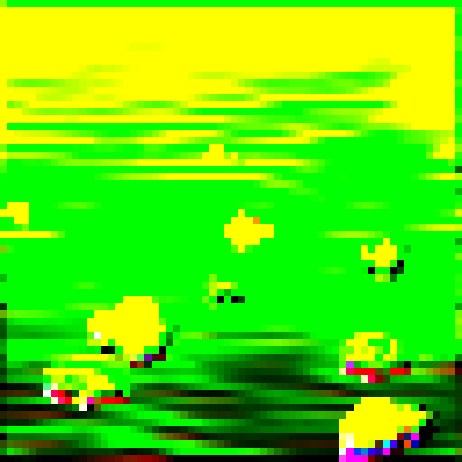}
	\includegraphics[scale=0.17]{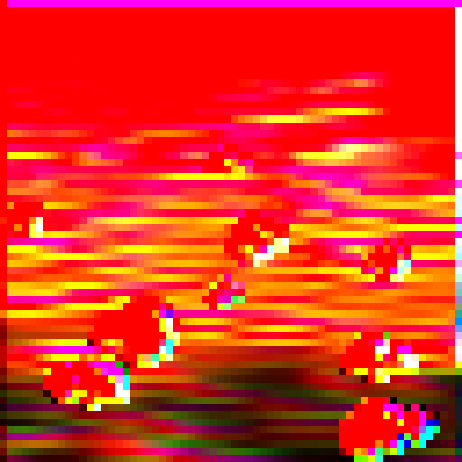}\includegraphics[scale=0.17]{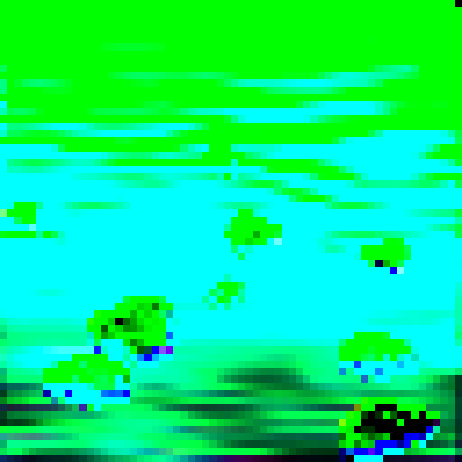}\includegraphics[scale=0.17]{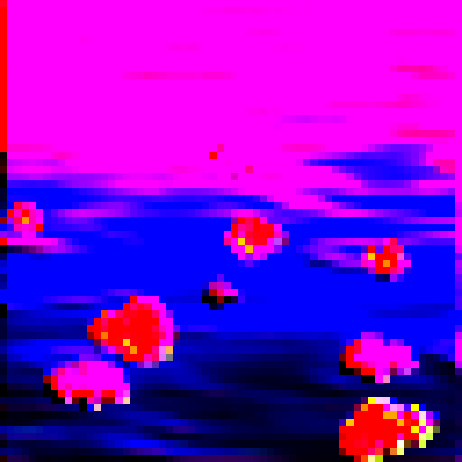}\includegraphics[scale=0.17]{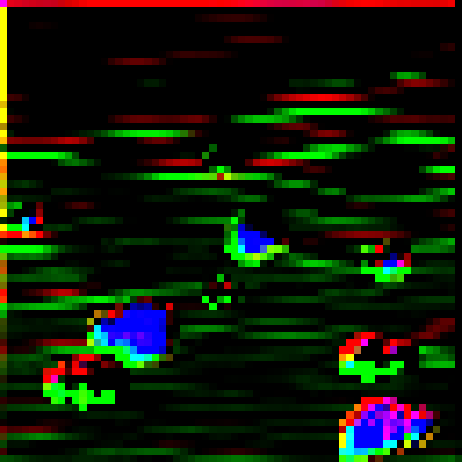}\includegraphics[scale=0.17]{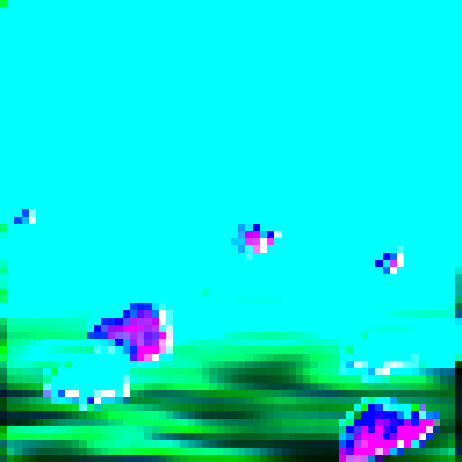}\includegraphics[scale=0.17]{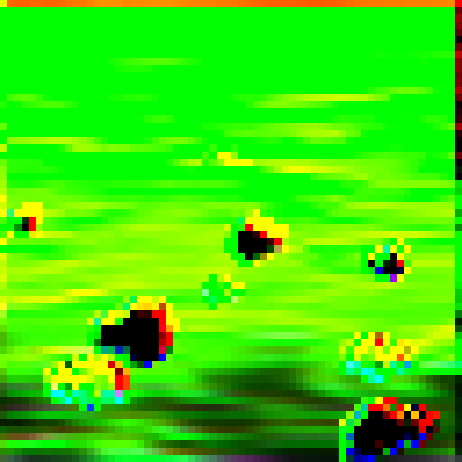}
	\caption{Adversarial samples generated by random CNNs.}\label{adversarial samples 1}
\end{figure}

\begin{figure}[!h]
	\centering
	\includegraphics[scale=0.17]{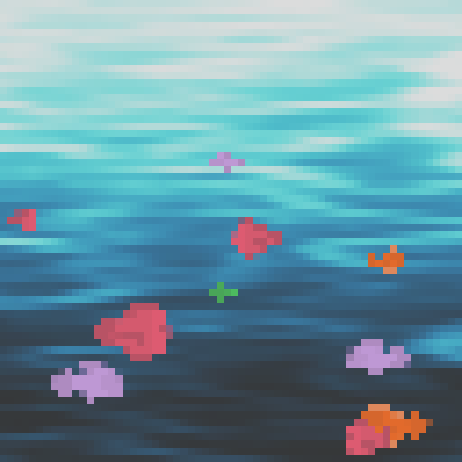}\includegraphics[scale=0.17]{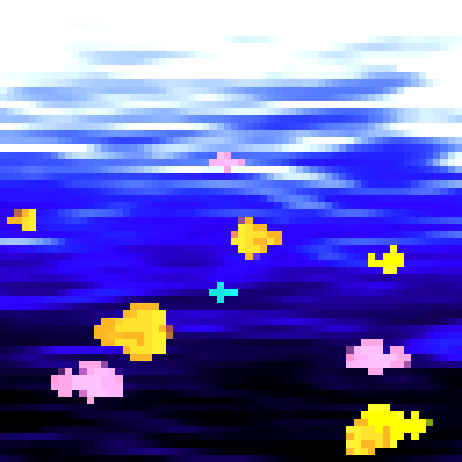}\includegraphics[scale=0.17]{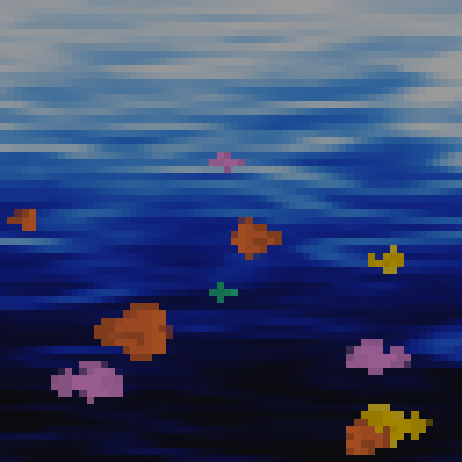}\includegraphics[scale=0.17]{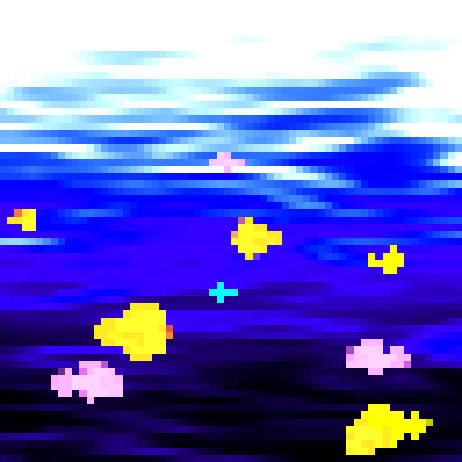}\includegraphics[scale=0.17]{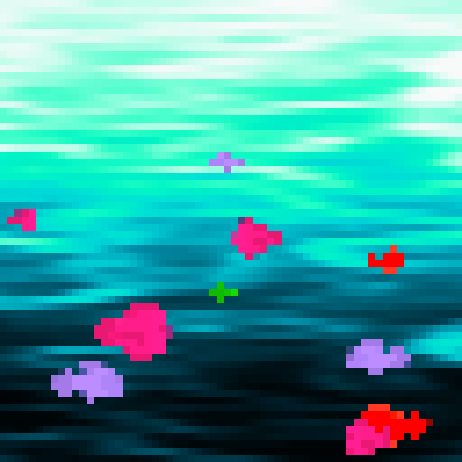}\includegraphics[scale=0.17]{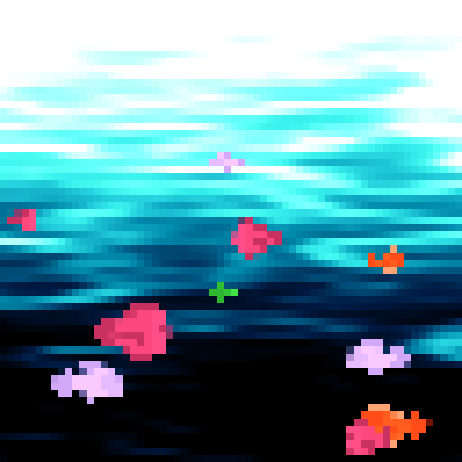}
	\includegraphics[scale=0.17]{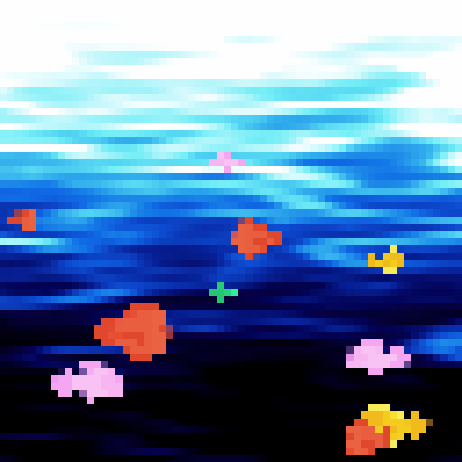}\includegraphics[scale=0.17]{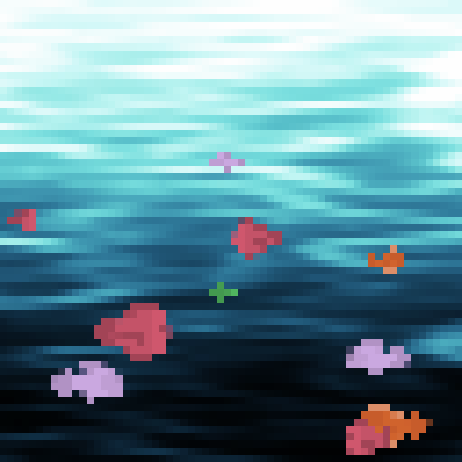}\includegraphics[scale=0.17]{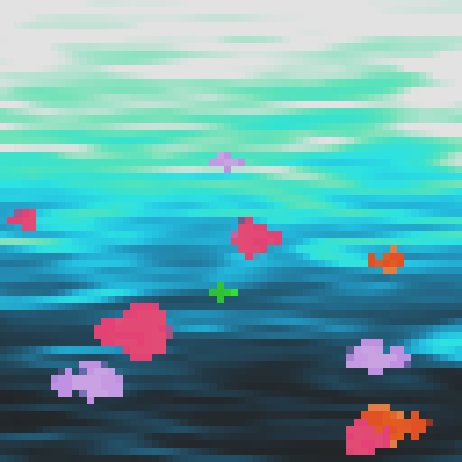}\includegraphics[scale=0.17]{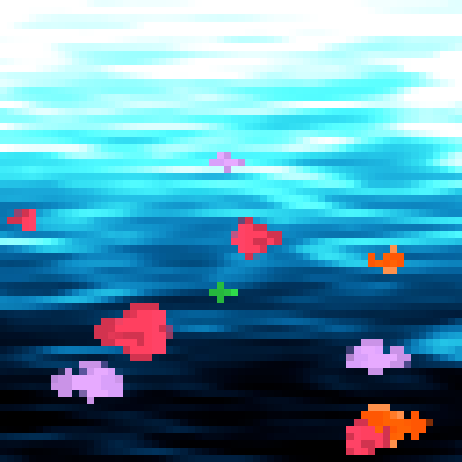}\includegraphics[scale=0.17]{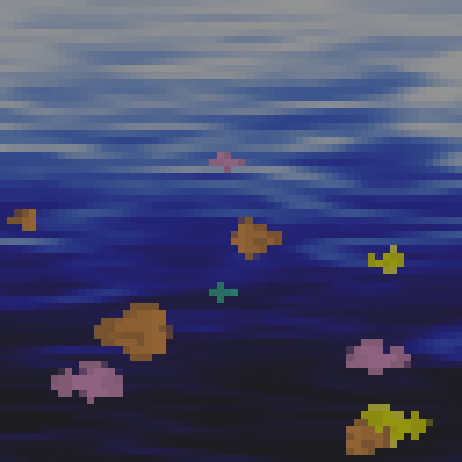}\includegraphics[scale=0.17]{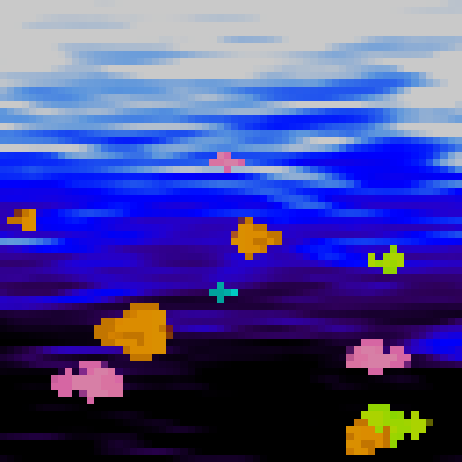}
	\caption{Adversarial samples generated by different brightness, contrast, saturation, and hue.}\label{adversarial samples 2}
\end{figure}

\newpage
\section{Proofs} \label{Proofs}
Let's start with some useful lemmas.
\begin{lemma}[Performance difference]\label{lemma 1}
	Let $\mu_f(\cdot|s_t)=\pi(\cdot|f(s_t))$ and $\tilde{\mu}_f(\cdot|s_t)=\tilde{\pi}(\cdot|f(s_t))$, define training and generalization performance as
	\begin{equation}
		\eta(\pi)=\frac{1}{1-\gamma}\mathop{\mathbb{E}}_{\substack{f\sim p_{\mathrm{train}}(\cdot)\\s\sim d^{\mu_f}(\cdot)\\a\sim\mu_f(\cdot|s)}}\left[r(s,a)\right],\enspace\zeta(\pi)=\frac{1}{1-\gamma}\mathop{\mathbb{E}}_{\substack{f\sim p(\cdot)\\s\sim d^{\mu_f}(\cdot)\\a\sim\mu_f(\cdot|s)}}\left[r(s,a)\right].
	\end{equation}
	Then the differences in training and generalization performance can be expressed as
	\begin{equation}
		\eta(\tilde{\pi})-\eta(\pi)=\frac{1}{1-\gamma}\mathop{\mathbb{E}}_{\substack{f\sim p_{\mathrm{train}}(\cdot)\\s\sim d^{\tilde{\mu}_f}(\cdot)\\a\sim\tilde{\mu}_f(\cdot|s)}}\left[A^{\mu_f}(s,a)\right],\enspace\zeta(\tilde{\pi})-\zeta(\pi)=\frac{1}{1-\gamma}\mathop{\mathbb{E}}_{\substack{f\sim p(\cdot)\\s\sim d^{\tilde{\mu}_f}(\cdot)\\a\sim\tilde{\mu}_f(\cdot|s)}}\left[A^{\mu_f}(s,a)\right].
	\end{equation}
\end{lemma}
\begin{proof}
	This result can be directly derived from \citet{kakade2002approximately}.
\end{proof}

\begin{lemma}\label{lemma 2}
	The divergence between two normalized discounted visitation distribution, $\Vert d^{\tilde{\mu}}-d^{\mu}\Vert_1$, is bounded by an average divergence of $\tilde{\mu}$ and $\mu$:
	\begin{equation}
		\Vert d^{\tilde{\mu}}-d^{\mu}\Vert_1\leq\frac{\gamma}{1-\gamma}\mathop{\mathbb{E}}_{s\sim d^{\mu}(\cdot)}\left[\Vert\tilde{\mu}-\mu\Vert_1\right]=\frac{2\gamma}{1-\gamma}\mathop{\mathbb{E}}_{s\sim d^{\mu}(\cdot)}\left[D_{\mathrm{TV}}(\tilde{\mu}\Vert\mu)[s]\right],
	\end{equation}
	where $D_{\mathrm{TV}}(\tilde{\mu}\Vert\mu)[s]=\frac{1}{2}\sum_{a\in\mathcal{A}}\vert\tilde{\mu}(a|s)-\mu(a|s)\vert$ represents the Total Variation (TV) distance.
\end{lemma}
\begin{proof}
	See \citet{achiam2017constrained}.
\end{proof}

\begin{lemma}\label{lemma 3}
	Given any state $s\in\mathcal{S}$, any two policies $\tilde{\mu}$ and $\mu$, the average advantage, $\mathop{\mathbb{E}}_{a\sim\tilde{\mu}(\cdot|s)}\left[A^{\mu}(s,a)\right]$, is bounded by
	\begin{equation}
		\left\vert\mathbb{E}_{a\sim\tilde{\mu}(\cdot|s)}\left[A^{\mu}(s,a)\right]\right\vert\leq2D_{\mathrm{TV}}(\tilde{\mu}\Vert\mu)[s]\cdot\max_a\left\vert A^{\mu}(s,a)\right\vert.
	\end{equation}
\end{lemma}
\begin{proof}
	Note that
	\begin{equation}
		\begin{split}
			\mathbb{E}_{a\sim\mu(\cdot|s)}\left[A^{\mu}(s,a)\right]=&\mathbb{E}_{a\sim\mu(\cdot|s)}\left[Q^{\mu}(s,a)-V^{\mu}(s)\right] \\
			=&\mathbb{E}_{a\sim\mu(\cdot|s)}\left[Q^{\mu}(s,a)\right]-V^{\mu}(s) \\
			=&V^{\mu}(s)-V^{\mu}(s) \\
			=&0,\\
		\end{split}
	\end{equation}
	thus,
	\begin{equation}
		\begin{split}
			\left\vert\mathbb{E}_{a\sim\tilde{\mu}(\cdot|s)}\left[A^{\mu}(s,a)\right]\right\vert&=\left\vert\mathbb{E}_{a\sim\tilde{\mu}(\cdot|s)}\left[A^{\mu}(s,a)\right]-\mathbb{E}_{a\sim\mu(\cdot|s)}\left[A^{\mu}(s,a)\right]\right\vert\\
			&\leq\left\Vert\tilde{\mu}-\mu\right\Vert_1\cdot\left\Vert A^{\mu}(s,a)\right\Vert_{\infty}\\
			&=2D_{\mathrm{TV}}(\tilde{\mu}\Vert\mu)[s]\cdot\max_a\left\vert A^{\mu}(s,a)\right\vert.\\
		\end{split}
	\end{equation}
	This is a widely used trick \citep{schulman2015trust, zhuang2023behavior, gan2024transductive}.
\end{proof}

In addition, using the above lemmas, the following corollary can be obtained, which will be repeatedly used in our proof.
\begin{corollary}\label{corollary 1}
	Given any two policies, $\tilde{\mu}$ and $\mu$, the following bound holds:
	\begin{equation}
		\left\vert\mathop{\mathbb{E}}_{\substack{s\sim d^{\tilde{\mu}}(\cdot)\\a\sim\tilde{\mu}(\cdot|s)}}\left[A^{\mu}(s,a)\right]-\mathop{\mathbb{E}}_{\substack{s\sim d^{\mu}(\cdot)\\a\sim\tilde{\mu}(\cdot|s)}}\left[A^{\mu}(s,a)\right]\right\vert\leq\frac{2\epsilon\gamma}{1-\gamma}\mathop{\mathbb{E}}_{s\sim d^{\mu}(\cdot)}\left[D_{\mathrm{TV}}(\tilde{\mu}\Vert\mu)[s]\right],
	\end{equation}
	where $\epsilon=\max_s\left\vert\mathbb{E}_{a\sim\tilde{\mu}(\cdot|s)}\left[A^{\mu}(s,a)\right]\right\vert$.
\end{corollary}
\begin{proof}
	We rewrite the expectation as
	\begin{equation}
		\left\vert\mathop{\mathbb{E}}_{\substack{s\sim d^{\tilde{\mu}}(\cdot)\\a\sim\tilde{\mu}(\cdot|s)}}\left[A^{\mu}(s,a)\right]-\mathop{\mathbb{E}}_{\substack{s\sim d^{\mu}(\cdot)\\a\sim\tilde{\mu}(\cdot|s)}}\left[A^{\mu}(s,a)\right]\right\vert=\left\vert\mathop{\mathbb{E}}_{s\sim d^{\tilde{\mu}}(\cdot)}\left\{\mathop{\mathbb{E}}_{a\sim\tilde{\mu}(\cdot|s)}\left[A^{\mu}(s,a)\right]\right\}-\mathop{\mathbb{E}}_{s\sim d^{\mu}(\cdot)}\left\{\mathop{\mathbb{E}}_{a\sim\tilde{\mu}(\cdot|s)}\left[A^{\mu}(s,a)\right]\right\}\right\vert,
	\end{equation}
	where the expectation $\mathop{\mathbb{E}}_{a\sim\tilde{\mu}(\cdot|s)}\left[A^{\mu}(s,a)\right]$ is a function of $s$, then
	\begin{equation}
		\left\vert\mathop{\mathbb{E}}_{s\sim d^{\tilde{\mu}}(\cdot)}\left\{\mathop{\mathbb{E}}_{a\sim\tilde{\mu}(\cdot|s)}\left[A^{\mu}(s,a)\right]\right\}-\mathop{\mathbb{E}}_{s\sim d^{\mu}(\cdot)}\left\{\mathop{\mathbb{E}}_{a\sim\tilde{\mu}(\cdot|s)}\left[A^{\mu}(s,a)\right]\right\}\right\vert\leq\left\Vert d^{\tilde{\mu}}-d^{\mu}\right\Vert_1\cdot\left\Vert\mathop{\mathbb{E}}_{a\sim\tilde{\mu}(\cdot|s)}\left[A^{\mu}(s,a)\right]\right\Vert_{\infty}.
	\end{equation}
	Next, according to Lemma \ref{lemma 2}, we have
	\begin{equation}
		\left\Vert d^{\tilde{\mu}}-d^{\mu}\right\Vert_1\cdot\left\Vert\mathop{\mathbb{E}}_{a\sim\tilde{\mu}(\cdot|s)}\left[A^{\mu}(s,a)\right]\right\Vert_{\infty}=\left\Vert d^{\tilde{\mu}}-d^{\mu}\right\Vert_1\cdot\epsilon\leq\frac{2\epsilon\gamma}{1-\gamma}\mathop{\mathbb{E}}_{s\sim d^{\mu}(\cdot)}\left[D_{\mathrm{TV}}(\tilde{\mu}\Vert\mu)[s]\right],
	\end{equation}
	concluding the proof.
\end{proof}

\subsection{Proof of Lemma \ref{lemma}}\label{proof lemma}
\textbf{Lemma \ref{lemma}.} \textit{
	For any given policy $\pi$, define its underlying policy as $\mu_f(\cdot|s_t)=\pi(\cdot|f(s_t))$, then
	\begin{equation}
		\eta(\pi)=\frac{1}{1-\gamma}\mathop{\mathbb{E}}_{\substack{f\sim p_{\mathrm{train}}(\cdot)\\s\sim d^{\mu_f}(\cdot)\\a\sim\mu_f(\cdot|s)}}\left[r(s,a)\right],\enspace\zeta(\pi)=\frac{1}{1-\gamma}\mathop{\mathbb{E}}_{\substack{f\sim p(\cdot)\\s\sim d^{\mu_f}(\cdot)\\a\sim\mu_f(\cdot|s)}}\left[r(s,a)\right].
	\end{equation}
}
\begin{proof}
	According to the definition of training and generalization performance in (\ref{training and generalization performance}), we have
	\begin{equation}
		\eta(\pi)=\mathbb{E}_{f\sim p_{\mathrm{train}}(\cdot),\tau_f\sim\pi}\left[\sum_{t=0}^{\infty}\gamma^tr_f(o_t^f,a_t)\right],\enspace\zeta(\pi)=\mathbb{E}_{f\sim p(\cdot),\tau_f\sim\pi}\left[\sum_{t=0}^{\infty}\gamma^tr_f(o_t^f,a_t)\right].
	\end{equation}
	To prove Lemma \ref{lemma}, we only need to show that for any given $f\in\mathcal{F}$, the following equation holds:
	\begin{equation}
		\frac{1}{1-\gamma}\mathop{\mathbb{E}}_{\substack{s\sim d^{\mu_f}(\cdot)\\a\sim\mu_f(\cdot|s)}}\left[r(s,a)\right]=\mathbb{E}_{\tau_f\sim\pi}\left[\sum_{t=0}^{\infty}\gamma^tr_f(o_t^f,a_t)\right].
	\end{equation}
	According to the definition of the normalized discounted visitation distribution $d^{\mu}(s)=(1-\gamma)\sum_{t=0}^{\infty}\gamma^t\mathbb{P}(s_t=s|\mu)$, we have
	\begin{equation}
		\begin{split}
			\frac{1}{1-\gamma}\mathop{\mathbb{E}}_{\substack{s\sim d^{\mu_f}(\cdot)\\a\sim\mu_f(\cdot|s)}}\left[r(s,a)\right]&=\frac{1}{1-\gamma}\sum_{s\in\mathcal{S}}(1-\gamma)\sum_{t=0}^{\infty}\gamma^t\mathbb{P}(s_t=s|\mu_f)\sum_{a\in\mathcal{A}}\mu_f(a|s)\cdot r(s,a)\\
			&=\sum_{t=0}^{\infty}\sum_{s\in\mathcal{S}}\mathbb{P}(s_t=s|\mu_f)\sum_{a\in\mathcal{A}}\mu_f(a|s)\cdot\gamma^tr(s,a)\\
		\end{split}
	\end{equation}
	Next, according to Assumption \ref{assumption 2}, we have
	\begin{equation}
		\begin{split}
			\frac{1}{1-\gamma}\mathop{\mathbb{E}}_{\substack{s\sim d^{\mu_f}(\cdot)\\a\sim\mu_f(\cdot|s)}}\left[r(s,a)\right]&=\sum_{t=0}^{\infty}\sum_{s\in\mathcal{S}}\mathbb{P}(s_t=s|\mu_f)\sum_{a\in\mathcal{A}}\mu_f(a|s)\cdot\gamma^tr(s,a)\\
			&=\sum_{t=0}^{\infty}\sum_{s\in\mathcal{S}}\mathbb{P}(f(s_t)=f(s)|\mu_f)\sum_{a\in\mathcal{A}}\pi(a|f(s))\cdot\gamma^tr_f(f(s),a)\\
			&\overset{f(s)=o^f,f(s_t)=o_t^f}{=}\sum_{t=0}^{\infty}\sum_{o^f\in\mathcal{O}_f}\mathbb{P}(o_t^f=o^f|\pi)\sum_{a\in\mathcal{A}}\pi(a|o^f)\cdot\gamma^tr_f(o^f,a)\\
			&=\mathbb{E}_{\tau_f\sim\pi}\left[\sum_{t=0}^{\infty}\gamma^tr_f(o_t^f,a_t)\right],\\
		\end{split}
	\end{equation}
	concluding the proof.
\end{proof}

\subsection{Proof of Theorem \ref{theorem 2}}\label{proof 2}
\textbf{Theorem \ref{theorem 2}.} \textit{
	Given any two policies, $\tilde{\pi}$ and $\pi$, the following bound holds:
	\begin{equation}
		\begin{split}
			\zeta(\tilde{\pi})&\geq L_{\pi}(\tilde{\pi})-\frac{2r_{\max}(1-Z)}{1-\gamma}-\frac{2\gamma\epsilon_{\mathrm{train}}}{(1-\gamma)^2}\mathop{\mathbb{E}}_{\substack{f\sim p_{\mathrm{train}}(\cdot)\\s\sim d^{\mu_f}(\cdot)}}\left[D_{\mathrm{TV}}(\tilde{\mu}_f\Vert\mu_f)[s]\right]\\
			&-\frac{2\delta_{\mathrm{train}}(1-Z)}{1-\gamma}\mathop{\mathbb{E}}_{\substack{f\sim p_{\mathrm{train}}(\cdot)\\s\sim d^{\tilde{\mu}_f}(\cdot)}}\left[D_{\mathrm{TV}}(\tilde{\mu}_f\Vert\mu_f)[s]\right]-\frac{2\delta_{\mathrm{eval}}(1-Z)}{1-\gamma}\mathop{\mathbb{E}}_{\substack{f\sim p_{\mathrm{eval}}(\cdot)\\s\sim d^{\tilde{\mu}_f}(\cdot)}}\left[D_{\mathrm{TV}}(\tilde{\mu}_f\Vert\mu_f)[s]\right].\\
		\end{split}
	\end{equation}
}
\begin{proof}
	Let's start with the first-order approximation of the training performance \citep{schulman2015trust}, denote it as
	\begin{equation}\label{first-order}
		L_{\pi}(\tilde{\pi})=\eta(\pi)+\frac{1}{1-\gamma}\mathop{\mathbb{E}}_{\substack{f\sim p_{\mathrm{train}}(\cdot)\\s\sim d^{\mu_f}(\cdot)\\a\sim\tilde{\mu}_f(\cdot|s)}}\left[A^{\mu_f}(s,a)\right].
	\end{equation}
	Then, we are trying to bound the difference between $\zeta(\tilde{\pi})$ and $L_{\pi}(\tilde{\pi})$, according to Lemma \ref{lemma 1}, that is,
	\begin{equation}\label{first and second terms}
		\begin{split}
			&\left\vert\zeta(\tilde{\pi})-L_{\pi}(\tilde{\pi})\right\vert\\
			=&\left\vert\zeta(\pi)-\eta(\pi)+\frac{1}{1-\gamma}\mathop{\mathbb{E}}_{\substack{f\sim p(\cdot)\\s\sim d^{\tilde{\mu}_f}(\cdot)\\a\sim\tilde{\mu}_f(\cdot|s)}}\left[A^{\mu_f}(s,a)\right]-\frac{1}{1-\gamma}\mathop{\mathbb{E}}_{\substack{f\sim p_{\mathrm{train}}(\cdot)\\s\sim d^{\mu_f}(\cdot)\\a\sim\tilde{\mu}_f(\cdot|s)}}\left[A^{\mu_f}(s,a)\right]\right\vert \\
			=&\frac{1}{1-\gamma}\left\vert\mathop{\mathbb{E}}_{\substack{f\sim p(\cdot)\\s\sim d^{\mu_f}(\cdot)\\a\sim\mu_f(\cdot|s)}}\left[r(s,a)\right]-\mathop{\mathbb{E}}_{\substack{f\sim p_{\mathrm{train}}(\cdot)\\s\sim d^{\mu_f}(\cdot)\\a\sim\mu_f(\cdot|s)}}\left[r(s,a)\right]+\mathop{\mathbb{E}}_{\substack{f\sim p(\cdot)\\s\sim d^{\tilde{\mu}_f}(\cdot)\\a\sim\tilde{\mu}_f(\cdot|s)}}\left[A^{\mu_f}(s,a)\right]-\mathop{\mathbb{E}}_{\substack{f\sim p_{\mathrm{train}}(\cdot)\\s\sim d^{\mu_f}(\cdot)\\a\sim\tilde{\mu}_f(\cdot|s)}}\left[A^{\mu_f}(s,a)\right]\right\vert \\
			\leq&\frac{1}{1-\gamma}\left\{\left\vert\mathop{\mathbb{E}}_{\substack{f\sim p(\cdot)\\s\sim d^{\mu_f}(\cdot)\\a\sim\mu_f(\cdot|s)}}\left[r(s,a)\right]-\mathop{\mathbb{E}}_{\substack{f\sim p_{\mathrm{train}}(\cdot)\\s\sim d^{\mu_f}(\cdot)\\a\sim\mu_f(\cdot|s)}}\left[r(s,a)\right]\right\vert+\left\vert\mathop{\mathbb{E}}_{\substack{f\sim p(\cdot)\\s\sim d^{\tilde{\mu}_f}(\cdot)\\a\sim\tilde{\mu}_f(\cdot|s)}}\left[A^{\mu_f}(s,a)\right]-\mathop{\mathbb{E}}_{\substack{f\sim p_{\mathrm{train}}(\cdot)\\s\sim d^{\mu_f}(\cdot)\\a\sim\tilde{\mu}_f(\cdot|s)}}\left[A^{\mu_f}(s,a)\right]\right\vert\right\}. \\
		\end{split}
	\end{equation}
	We can bound these two terms separately. Simplifying the notation, denote $g(f)=\mathop{\mathbb{E}}_{s\sim d^{\mu_f}(\cdot),a\sim\mu_f(\cdot|s)}\left[r(s,a)\right]$, we can thus rewrite the first term as
	\begin{equation}
		\left\vert\mathop{\mathbb{E}}_{\substack{f\sim p(\cdot)\\s\sim d^{\mu_f}(\cdot)\\a\sim\mu_f(\cdot|s)}}\left[r(s,a)\right]-\mathop{\mathbb{E}}_{\substack{f\sim p_{\mathrm{train}}(\cdot)\\s\sim d^{\mu_f}(\cdot)\\a\sim\mu_f(\cdot|s)}}\left[r(s,a)\right]\right\vert=\left\vert\mathop{\mathbb{E}}_{f\sim p(\cdot)}\left[g(f)\right]-\mathop{\mathbb{E}}_{f\sim p_{\mathrm{train}}(\cdot)}\left[g(f)\right]\right\vert,
	\end{equation}
	then
	\begin{equation}\label{g(f) 1}
		\left\vert\mathop{\mathbb{E}}_{f\sim p(\cdot)}\left[g(f)\right]-\mathop{\mathbb{E}}_{f\sim p_{\mathrm{train}}(\cdot)}\left[g(f)\right]\right\vert=\left\vert\int_{\mathcal{F}}p(f)\cdot g(f)\mathrm{d}f-\int_{\mathcal{F}_{\mathrm{train}}}p_{\mathrm{train}}(f)\cdot g(f)\mathrm{d}f\right\vert.
	\end{equation}
	Next, according to Assumption \ref{assumption 3},
	\begin{equation}\label{g(f) 2}
		\begin{split}
			&\left\vert\int_{\mathcal{F}}p(f)\cdot g(f)\mathrm{d}f-\int_{\mathcal{F}_{\mathrm{train}}}p_{\mathrm{train}}(f)\cdot g(f)\mathrm{d}f\right\vert=\left\vert\int_{\mathcal{F}}p(f)\cdot g(f)\mathrm{d}f-\int_{\mathcal{F}_{\mathrm{train}}}\frac{p(f)}{Z}\cdot g(f)\mathrm{d}f\right\vert \\
			=&\left\vert\int_{\mathcal{F}_{\mathrm{train}}}p(f)\cdot g(f)\mathrm{d}f-\int_{\mathcal{F}_{\mathrm{train}}}\frac{p(f)}{Z}\cdot g(f)\mathrm{d}f+\int_{\mathcal{F}-\mathcal{F}_{\mathrm{train}}}p(f)\cdot g(f)\mathrm{d}f\right\vert \\
			=&\left\vert\int_{\mathcal{F}_{\mathrm{train}}}\frac{Z-1}{Z}p(f)\cdot g(f)\mathrm{d}f+\int_{\mathcal{F}-\mathcal{F}_{\mathrm{train}}}p(f)\cdot g(f)\mathrm{d}f\right\vert, \\
		\end{split}
	\end{equation}
	where $Z=\int_{\mathcal{F}_{\mathrm{train}}}p(f)\mathrm{d}f\leq1$, thus,
	\begin{equation}\label{g(f) 3}
		\begin{split}
			&\left\vert\int_{\mathcal{F}_{\mathrm{train}}}\frac{Z-1}{Z}p(f)\cdot g(f)\mathrm{d}f+\int_{\mathcal{F}-\mathcal{F}_{\mathrm{train}}}p(f)\cdot g(f)\mathrm{d}f\right\vert \\
			\leq&\left\vert\int_{\mathcal{F}_{\mathrm{train}}}\frac{Z-1}{Z}p(f)\cdot g(f)\mathrm{d}f\right\vert+\left\vert\int_{\mathcal{F}-\mathcal{F}_{\mathrm{train}}}p(f)\cdot g(f)\mathrm{d}f\right\vert \\
			\leq&\frac{1-Z}{Z}\left\vert\int_{\mathcal{F}_{\mathrm{train}}}p(f)\cdot g(f)\mathrm{d}f\right\vert+\left\vert\int_{\mathcal{F}-\mathcal{F}_{\mathrm{train}}}p(f)\cdot g(f)\mathrm{d}f\right\vert. \\
		\end{split}
	\end{equation}
	Meanwhile,
	\begin{equation}
		\begin{split}
			\left\vert g(f)\right\vert=\left\vert\mathop{\mathbb{E}}_{\substack{s\sim d^{\mu_f}(\cdot)\\a\sim\mu_f(\cdot|s)}}\left[r(s,a)\right]\right\vert=&\left\vert\sum_{s\in\mathcal{S}}(1-\gamma)\sum_{t=0}^{\infty}\gamma^t\mathbb{P}(s_t=s|\mu_f)\sum_{a\in\mathcal{A}}\mu_f(a|s)\cdot r(s,a)\right\vert\\
			\leq&(1-\gamma)\sum_{t=0}^{\infty}\sum_{s\in\mathcal{S}}\mathbb{P}(s_t=s|\mu_f)\sum_{a\in\mathcal{A}}\mu_f(a|s)\cdot\gamma^t\left\vert r(s,a)\right\vert \\
			\leq&(1-\gamma)\sum_{t=0}^{\infty}\gamma^tr_{\max}=r_{\max}, \\
		\end{split}
	\end{equation}
	where $r_{\max}=\max_{s,a}\left\vert r(s,a)\right\vert$, then we can bound the first term as
	\begin{equation}\label{g(f) 4}
		\begin{split}
			\left\vert\mathop{\mathbb{E}}_{\substack{f\sim p(\cdot)\\s\sim d^{\mu_f}(\cdot)\\a\sim\mu_f(\cdot|s)}}\left[r(s,a)\right]-\mathop{\mathbb{E}}_{\substack{f\sim p_{\mathrm{train}}(\cdot)\\s\sim d^{\mu_f}(\cdot)\\a\sim\mu_f(\cdot|s)}}\left[r(s,a)\right]\right\vert\leq&\frac{1-Z}{Z}\left\vert\int_{\mathcal{F}_{\mathrm{train}}}p(f)\cdot g(f)\mathrm{d}f\right\vert+\left\vert\int_{\mathcal{F}-\mathcal{F}_{\mathrm{train}}}p(f)\cdot g(f)\mathrm{d}f\right\vert\\
			\leq&\frac{1-Z}{Z}\int_{\mathcal{F}_{\mathrm{train}}}p(f)\cdot \left\vert g(f)\right\vert\mathrm{d}f+\int_{\mathcal{F}-\mathcal{F}_{\mathrm{train}}}p(f)\cdot\left\vert g(f)\right\vert\mathrm{d}f\\
			\leq&\frac{(1-Z)r_{\max}}{Z}\int_{\mathcal{F}_{\mathrm{train}}}p(f)\mathrm{d}f+r_{\max}\int_{\mathcal{F}-\mathcal{F}_{\mathrm{train}}}p(f)\mathrm{d}f\\
			=&\frac{(1-Z)r_{\max}}{Z}\cdot Z+r_{\max}\cdot(1-Z)=2r_{\max}(1-Z).\\
		\end{split}
	\end{equation}
	
	Now we are trying to bound the second term, which can be expressed as
	\begin{equation}\label{second term}
		\begin{split}
			&\left\vert\mathop{\mathbb{E}}_{\substack{f\sim p(\cdot)\\s\sim d^{\tilde{\mu}_f}(\cdot)\\a\sim\tilde{\mu}_f(\cdot|s)}}\left[A^{\mu_f}(s,a)\right]-\mathop{\mathbb{E}}_{\substack{f\sim p_{\mathrm{train}}(\cdot)\\s\sim d^{\mu_f}(\cdot)\\a\sim\tilde{\mu}_f(\cdot|s)}}\left[A^{\mu_f}(s,a)\right]\right\vert\\
			=&\left\vert\mathop{\mathbb{E}}_{\substack{f\sim p(\cdot)\\s\sim d^{\tilde{\mu}_f}(\cdot)\\a\sim\tilde{\mu}_f(\cdot|s)}}\left[A^{\mu_f}(s,a)\right]-\mathop{\mathbb{E}}_{\substack{f\sim p_{\mathrm{train}}(\cdot)\\s\sim d^{\tilde{\mu}_f}(\cdot)\\a\sim\tilde{\mu}_f(\cdot|s)}}\left[A^{\mu_f}(s,a)\right]+\mathop{\mathbb{E}}_{\substack{f\sim p_{\mathrm{train}}(\cdot)\\s\sim d^{\tilde{\mu}_f}(\cdot)\\a\sim\tilde{\mu}_f(\cdot|s)}}\left[A^{\mu_f}(s,a)\right]-\mathop{\mathbb{E}}_{\substack{f\sim p_{\mathrm{train}}(\cdot)\\s\sim d^{\mu_f}(\cdot)\\a\sim\tilde{\mu}_f(\cdot|s)}}\left[A^{\mu_f}(s,a)\right]\right\vert \\
			\leq&\underbrace{\left\vert\mathop{\mathbb{E}}_{\substack{f\sim p(\cdot)\\s\sim d^{\tilde{\mu}_f}(\cdot)\\a\sim\tilde{\mu}_f(\cdot|s)}}\left[A^{\mu_f}(s,a)\right]-\mathop{\mathbb{E}}_{\substack{f\sim p_{\mathrm{train}}(\cdot)\\s\sim d^{\tilde{\mu}_f}(\cdot)\\a\sim\tilde{\mu}_f(\cdot|s)}}\left[A^{\mu_f}(s,a)\right]\right\vert}_{\text{denote as } \Phi}+\underbrace{\left\vert\mathop{\mathbb{E}}_{\substack{f\sim p_{\mathrm{train}}(\cdot)\\s\sim d^{\tilde{\mu}_f}(\cdot)\\a\sim\tilde{\mu}_f(\cdot|s)}}\left[A^{\mu_f}(s,a)\right]-\mathop{\mathbb{E}}_{\substack{f\sim p_{\mathrm{train}}(\cdot)\\s\sim d^{\mu_f}(\cdot)\\a\sim\tilde{\mu}_f(\cdot|s)}}\left[A^{\mu_f}(s,a)\right]\right\vert}_{\text{denote as } \Psi}. \\
		\end{split}
	\end{equation}
	Using Corollary \ref{corollary 1}, $\Psi$ can be bounded by
	\begin{equation}
		\begin{split}
			\Psi&=\left\vert\mathop{\mathbb{E}}_{f\sim p_{\mathrm{train}}(\cdot)}\left\{\mathop{\mathbb{E}}_{\substack{s\sim d^{\tilde{\mu}_f}(\cdot)\\a\sim\tilde{\mu}_f(\cdot|s)}}\left[A^{\mu_f}(s,a)\right]-\mathop{\mathbb{E}}_{\substack{s\sim d^{\mu_f}(\cdot)\\a\sim\tilde{\mu}_f(\cdot|s)}}\left[A^{\mu_f}(s,a)\right]\right\}\right\vert\\
			&\leq\mathop{\mathbb{E}}_{f\sim p_{\mathrm{train}}(\cdot)}\left\{\left\vert\mathop{\mathbb{E}}_{\substack{s\sim d^{\tilde{\mu}_f}(\cdot)\\a\sim\tilde{\mu}_f(\cdot|s)}}\left[A^{\mu_f}(s,a)\right]-\mathop{\mathbb{E}}_{\substack{s\sim d^{\mu_f}(\cdot)\\a\sim\tilde{\mu}_f(\cdot|s)}}\left[A^{\mu_f}(s,a)\right]\right\vert\right\}\\
			&\leq\mathop{\mathbb{E}}_{f\sim p_{\mathrm{train}}(\cdot)}\left\{\frac{2\epsilon\gamma}{1-\gamma}\mathop{\mathbb{E}}_{s\sim d^{\mu_f}(\cdot)}\left[D_{\mathrm{TV}}(\tilde{\mu}_f\Vert\mu_f)[s]\right]\right\},\\
		\end{split}
	\end{equation}
	where $\epsilon=\max_s\left\vert\mathbb{E}_{a\sim\tilde{\mu}_f(\cdot|s)}\left[A^{\mu_f}(s,a)\right]\right\vert$, denote $\epsilon_{\mathrm{train}}=\max_{f\in\mathcal{F}_{\mathrm{train}}}\left\{\epsilon\right\}$, we obtain
	\begin{equation}\label{Psi}
		\Psi\leq\frac{2\gamma\epsilon_{\mathrm{train}}}{1-\gamma}\mathop{\mathbb{E}}_{\substack{f\sim p_{\mathrm{train}}(\cdot)\\s\sim d^{\mu_f}(\cdot)}}\left[D_{\mathrm{TV}}(\tilde{\mu}_f\Vert\mu_f)[s]\right].
	\end{equation}
	Next, with a little abuse of notation $g(f)$, denote
	\begin{equation}
		g(f)=\mathop{\mathbb{E}}_{\substack{s\sim d^{\tilde{\mu}_f}(\cdot)\\a\sim\tilde{\mu}_f(\cdot|s)}}\left[A^{\mu_f}(s,a)\right],
	\end{equation}
	we can rewrite $\Phi$ as
	\begin{equation}
		\Phi=\left\vert\mathop{\mathbb{E}}_{f\sim p(\cdot)}\left[g(f)\right]-\mathop{\mathbb{E}}_{f\sim p_{\mathrm{train}}(\cdot)}\left[g(f)\right]\right\vert,
	\end{equation}
	then, similar to (\ref{g(f) 1}), (\ref{g(f) 2}), (\ref{g(f) 3}) and (\ref{g(f) 4}),
	\begin{equation}
		\Phi\leq\frac{1-Z}{Z}\int_{\mathcal{F}_{\mathrm{train}}}p(f)\cdot \left\vert g(f)\right\vert\mathrm{d}f+\int_{\mathcal{F}-\mathcal{F}_{\mathrm{train}}}p(f)\cdot\left\vert g(f)\right\vert\mathrm{d}f.
	\end{equation}
	According to Lemma \ref{lemma 3}, we can bound $g(f)$, which can be expressed as
	\begin{equation}
		g(f)=\mathop{\mathbb{E}}_{\substack{s\sim d^{\tilde{\mu}_f}(\cdot)\\a\sim\tilde{\mu}_f(\cdot|s)}}\left[A^{\mu_f}(s,a)\right]=\mathop{\mathbb{E}}_{s\sim d^{\tilde{\mu}_f}(\cdot)}\left\{\mathop{\mathbb{E}}_{a\sim\tilde{\mu}_f(\cdot|s)}\left[A^{\mu_f}(s,a)\right]\right\},
	\end{equation}
	thus,
	\begin{equation}
		\left\vert g(f)\right\vert\leq\mathop{\mathbb{E}}_{s\sim d^{\tilde{\mu}_f}(\cdot)}\left\{\left\vert\mathop{\mathbb{E}}_{a\sim\tilde{\mu}_f(\cdot|s)}\left[A^{\mu_f}(s,a)\right]\right\vert\right\}\leq\mathop{\mathbb{E}}_{s\sim d^{\tilde{\mu}_f}(\cdot)}\left\{2D_{\mathrm{TV}}(\tilde{\mu}_f\Vert\mu_f)[s]\cdot\max_a\left\vert A^{\mu_f}(s,a)\right\vert\right\}.
	\end{equation}
	Denote $\delta=\max_{s,a}\left\vert A^{\mu_f}(s,a)\right\vert$, then we have
	\begin{equation}
		\left\vert g(f)\right\vert\leq2\delta\mathop{\mathbb{E}}_{s\sim d^{\tilde{\mu}_f}(\cdot)}\left[D_{\mathrm{TV}}(\tilde{\mu}_f\Vert\mu_f)[s]\right],
	\end{equation}
	which means that
	\begin{equation}\label{Phi}
		\begin{split}
			\Phi\leq&\frac{1-Z}{Z}\int_{\mathcal{F}_{\mathrm{train}}}p(f)\cdot \left\vert g(f)\right\vert\mathrm{d}f+\int_{\mathcal{F}-\mathcal{F}_{\mathrm{train}}}p(f)\cdot\left\vert g(f)\right\vert\mathrm{d}f\\
			\leq&\frac{2\delta_{\mathrm{train}}(1-Z)}{Z}\int_{\mathcal{F}_{\mathrm{train}}}p(f)\cdot\mathop{\mathbb{E}}_{s\sim d^{\tilde{\mu}_f}(\cdot)}\left[D_{\mathrm{TV}}(\tilde{\mu}_f\Vert\mu_f)[s]\right]\mathrm{d}f\\
			&+2\delta_{\mathrm{eval}}\int_{\mathcal{F}-\mathcal{F}_{\mathrm{train}}}p(f)\cdot\mathop{\mathbb{E}}_{s\sim d^{\tilde{\mu}_f}(\cdot)}\left[D_{\mathrm{TV}}(\tilde{\mu}_f\Vert\mu_f)[s]\right]\mathrm{d}f\\
			=&2\delta_{\mathrm{train}}(1-Z)\int_{\mathcal{F}_{\mathrm{train}}}\frac{p(f)}{Z}\cdot\mathop{\mathbb{E}}_{s\sim d^{\tilde{\mu}_f}(\cdot)}\left[D_{\mathrm{TV}}(\tilde{\mu}_f\Vert\mu_f)[s]\right]\mathrm{d}f\\
			&+2\delta_{\mathrm{eval}}(1-Z)\int_{\mathcal{F}-\mathcal{F}_{\mathrm{train}}}\frac{p(f)}{1-Z}\cdot\mathop{\mathbb{E}}_{s\sim d^{\tilde{\mu}_f}(\cdot)}\left[D_{\mathrm{TV}}(\tilde{\mu}_f\Vert\mu_f)[s]\right]\mathrm{d}f\\
			=&2\delta_{\mathrm{train}}(1-Z)\mathop{\mathbb{E}}_{\substack{f\sim p_{\mathrm{train}}(\cdot)\\s\sim d^{\tilde{\mu}_f}(\cdot)}}\left[D_{\mathrm{TV}}(\tilde{\mu}_f\Vert\mu_f)[s]\right]+2\delta_{\mathrm{eval}}(1-Z)\mathop{\mathbb{E}}_{\substack{f\sim p_{\mathrm{eval}}(\cdot)\\s\sim d^{\tilde{\mu}_f}(\cdot)}}\left[D_{\mathrm{TV}}(\tilde{\mu}_f\Vert\mu_f)[s]\right],\\
		\end{split}
	\end{equation}
	where $\delta_{\mathrm{train}}=\max_{f\in\mathcal{F}_{\mathrm{train}}}\left\{\max_{s,a}\left\vert A^{\mu_f}(s,a)\right\vert\right\}$ and $\delta_{\mathrm{eval}}=\max_{f\in\mathcal{F}_{\mathrm{eval}}}\left\{\max_{s,a}\left\vert A^{\mu_f}(s,a)\right\vert\right\}$.
	
	Finally, combining (\ref{first and second terms}), (\ref{g(f) 4}), (\ref{second term}), (\ref{Psi}), and (\ref{Phi}), we have
	\begin{equation}
		\begin{split}
			\left\vert\zeta(\tilde{\pi})-L_{\pi}(\tilde{\pi})\right\vert&\leq\frac{2r_{\max}(1-Z)}{1-\gamma}+\frac{2\gamma\epsilon_{\mathrm{train}}}{(1-\gamma)^2}\mathop{\mathbb{E}}_{\substack{f\sim p_{\mathrm{train}}(\cdot)\\s\sim d^{\mu_f}(\cdot)}}\left[D_{\mathrm{TV}}(\tilde{\mu}_f\Vert\mu_f)[s]\right]\\
			&+\frac{2\delta_{\mathrm{train}}(1-Z)}{1-\gamma}\mathop{\mathbb{E}}_{\substack{f\sim p_{\mathrm{train}}(\cdot)\\s\sim d^{\tilde{\mu}_f}(\cdot)}}\left[D_{\mathrm{TV}}(\tilde{\mu}_f\Vert\mu_f)[s]\right]+\frac{2\delta_{\mathrm{eval}}(1-Z)}{1-\gamma}\mathop{\mathbb{E}}_{\substack{f\sim p_{\mathrm{eval}}(\cdot)\\s\sim d^{\tilde{\mu}_f}(\cdot)}}\left[D_{\mathrm{TV}}(\tilde{\mu}_f\Vert\mu_f)[s]\right],\\
		\end{split}
	\end{equation}
	thus, the generalization performance lower bound is
	\begin{equation}
		\begin{split}
			\zeta(\tilde{\pi})&\geq L_{\pi}(\tilde{\pi})-\frac{2r_{\max}(1-Z)}{1-\gamma}-\frac{2\gamma\epsilon_{\mathrm{train}}}{(1-\gamma)^2}\mathop{\mathbb{E}}_{\substack{f\sim p_{\mathrm{train}}(\cdot)\\s\sim d^{\mu_f}(\cdot)}}\left[D_{\mathrm{TV}}(\tilde{\mu}_f\Vert\mu_f)[s]\right]\\
			&-\frac{2\delta_{\mathrm{train}}(1-Z)}{1-\gamma}\mathop{\mathbb{E}}_{\substack{f\sim p_{\mathrm{train}}(\cdot)\\s\sim d^{\tilde{\mu}_f}(\cdot)}}\left[D_{\mathrm{TV}}(\tilde{\mu}_f\Vert\mu_f)[s]\right]-\frac{2\delta_{\mathrm{eval}}(1-Z)}{1-\gamma}\mathop{\mathbb{E}}_{\substack{f\sim p_{\mathrm{eval}}(\cdot)\\s\sim d^{\tilde{\mu}_f}(\cdot)}}\left[D_{\mathrm{TV}}(\tilde{\mu}_f\Vert\mu_f)[s]\right],\\
		\end{split}
	\end{equation}
	concluding the proof.
\end{proof}

\subsection{Proof of Theorem \ref{theorem 1}}\label{proof 1}
\textbf{Theorem \ref{theorem 1}.} \textit{
	Given any two policies, $\tilde{\pi}$ and $\pi$, the following bound holds:
	\begin{equation}
		\eta(\tilde{\pi})\geq L_{\pi}(\tilde{\pi})-\frac{2\gamma\epsilon_{\mathrm{train}}}{(1-\gamma)^2}\mathop{\mathbb{E}}_{\substack{f\sim p_{\mathrm{train}}(\cdot)\\s\sim d^{\mu_f}(\cdot)}}\left[D_{\mathrm{TV}}(\tilde{\mu}_f\Vert\mu_f)[s]\right].
	\end{equation}
}
\begin{proof}
	Since
	\begin{equation}
		\begin{split}
			\left\vert\eta(\tilde{\pi})-L_{\pi}(\tilde{\pi})\right\vert&=\frac{1}{1-\gamma}\left\vert\mathop{\mathbb{E}}_{\substack{f\sim p_{\mathrm{train}}(\cdot)\\s\sim d^{\tilde{\mu}_f}(\cdot)\\a\sim\tilde{\mu}_f(\cdot|s)}}\left[A^{\mu_f}(s,a)\right]-\mathop{\mathbb{E}}_{\substack{f\sim p_{\mathrm{train}}(\cdot)\\s\sim d^{\mu_f}(\cdot)\\a\sim\tilde{\mu}_f(\cdot|s)}}\left[A^{\mu_f}(s,a)\right]\right\vert=\frac{\Psi}{1-\gamma} \\
			&\leq\frac{2\gamma\epsilon_{\mathrm{train}}}{(1-\gamma)^2}\mathop{\mathbb{E}}_{\substack{f\sim p_{\mathrm{train}}(\cdot)\\s\sim d^{\mu_f}(\cdot)}}\left[D_{\mathrm{TV}}(\tilde{\mu}_f\Vert\mu_f)[s]\right],
		\end{split}
	\end{equation}
	thus,
	\begin{equation}
		\eta(\tilde{\pi})\geq L_{\pi}(\tilde{\pi})-\frac{2\gamma\epsilon_{\mathrm{train}}}{(1-\gamma)^2}\mathop{\mathbb{E}}_{\substack{f\sim p_{\mathrm{train}}(\cdot)\\s\sim d^{\mu_f}(\cdot)}}\left[D_{\mathrm{TV}}(\tilde{\mu}_f\Vert\mu_f)[s]\right],
	\end{equation}
	concluding the proof.
\end{proof}

\subsection{Proof of Theorem \ref{theorem 3}}\label{proof 3}
\textbf{Theorem \ref{theorem 3}.} \textit{
	Given any two policies, $\tilde{\pi}$ and $\pi$, the following bound holds:
	\begin{equation}
		\mathfrak{D}_1\leq\left(1+\frac{2\gamma\sigma_{\mathrm{train}}}{1-\gamma}\right)\mathfrak{D}_{\mathrm{train}},\enspace\mathfrak{D}_2\leq\left(1+\frac{2\gamma\sigma_{\mathrm{eval}}}{1-\gamma}\right)\underbrace{\mathop{\mathbb{E}}_{\substack{f\sim p_{\mathrm{eval}}(\cdot)\\s\sim d^{\mu_f}(\cdot)}}\left[D_{\mathrm{TV}}(\tilde{\mu}_f\Vert\mu_f)[s]\right]}_{\text{denote it as }\mathfrak{D}_{\mathrm{eval}}},
	\end{equation}
	where $\sigma_{\mathrm{train}}=\max_{f\in\mathcal{F}_{\mathrm{train}}}\left\{D_{\mathrm{TV}}^{\max}(\tilde{\mu}_f\Vert\mu_f)[s]\right\}$ and $\sigma_{\mathrm{eval}}=\max_{f\in\mathcal{F}_{\mathrm{eval}}}\left\{D_{\mathrm{TV}}^{\max}(\tilde{\mu}_f\Vert\mu_f)[s]\right\}$, $D_{\mathrm{TV}}^{\max}(\tilde{\mu}_f\Vert\mu_f)[s]$ represents $\max_{s}D_{\mathrm{TV}}(\tilde{\mu}_f\Vert\mu_f)[s]$.
}
\begin{proof}
	According to Lemma \ref{lemma 2}, we have
	\begin{equation}
		\begin{split}
			\left\vert\mathfrak{D}_1-\mathfrak{D}_{\mathrm{train}}\right\vert&=\left\vert\mathop{\mathbb{E}}_{\substack{f\sim p_{\mathrm{train}}(\cdot)\\s\sim d^{\tilde{\mu}_f}(\cdot)}}\left[D_{\mathrm{TV}}(\tilde{\mu}_f\Vert\mu_f)[s]\right]-\mathop{\mathbb{E}}_{\substack{f\sim p_{\mathrm{train}}(\cdot)\\s\sim d^{\mu_f}(\cdot)}}\left[D_{\mathrm{TV}}(\tilde{\mu}_f\Vert\mu_f)[s]\right]\right\vert\\
			&=\left\vert\mathop{\mathbb{E}}_{f\sim p_{\mathrm{train}}(\cdot)}\left\{\mathop{\mathbb{E}}_{s\sim d^{\tilde{\mu}_f}(\cdot)}\left[D_{\mathrm{TV}}(\tilde{\mu}_f\Vert\mu_f)[s]\right]-\mathop{\mathbb{E}}_{s\sim d^{\mu_f}(\cdot)}\left[D_{\mathrm{TV}}(\tilde{\mu}_f\Vert\mu_f)[s]\right]\right\}\right\vert\\
			&\leq\mathop{\mathbb{E}}_{f\sim p_{\mathrm{train}}(\cdot)}\left\{\left\vert\mathop{\mathbb{E}}_{s\sim d^{\tilde{\mu}_f}(\cdot)}\left[D_{\mathrm{TV}}(\tilde{\mu}_f\Vert\mu_f)[s]\right]-\mathop{\mathbb{E}}_{s\sim d^{\mu_f}(\cdot)}\left[D_{\mathrm{TV}}(\tilde{\mu}_f\Vert\mu_f)[s]\right]\right\vert\right\}\\
			&\leq\mathop{\mathbb{E}}_{f\sim p_{\mathrm{train}}(\cdot)}\left\{\left\Vert d^{\tilde{\mu}_f}-d^{\mu_f}\right\Vert_1\cdot\left\Vert D_{\mathrm{TV}}(\tilde{\mu}_f\Vert\mu_f)[s]\right\Vert_{\infty}\right\}\\
			&\leq\mathop{\mathbb{E}}_{f\sim p_{\mathrm{train}}(\cdot)}\left\{\frac{2\gamma}{1-\gamma}\mathop{\mathbb{E}}_{s\sim d^{\mu_f}(\cdot)}\left[D_{\mathrm{TV}}(\tilde{\mu}_f\Vert\mu_f)[s]\right]\cdot\max_{s}D_{\mathrm{TV}}(\tilde{\mu}_f\Vert\mu_f)[s]\right\}\\
			&\leq\frac{2\gamma\sigma_{\mathrm{train}}}{1-\gamma}\mathop{\mathbb{E}}_{\substack{f\sim p_{\mathrm{train}}(\cdot)\\s\sim d^{\mu_f}(\cdot)}}\left[D_{\mathrm{TV}}(\tilde{\mu}_f\Vert\mu_f)[s]\right]=\frac{2\gamma\sigma_{\mathrm{train}}}{1-\gamma}\cdot\mathfrak{D}_{\mathrm{train}},\\
		\end{split}
	\end{equation}
	as a result,
	\begin{equation}
		\mathfrak{D}_1\leq\left(1+\frac{2\gamma\sigma_{\mathrm{train}}}{1-\gamma}\right)\mathfrak{D}_{\mathrm{train}}.
	\end{equation}
	Similarly, using Lemma \ref{lemma 2} again, we have
	\begin{equation}
		\begin{split}
			\left\vert\mathfrak{D}_2-\mathfrak{D}_{\mathrm{eval}}\right\vert&=\left\vert\mathop{\mathbb{E}}_{\substack{f\sim p_{\mathrm{eval}}(\cdot)\\s\sim d^{\tilde{\mu}_f}(\cdot)}}\left[D_{\mathrm{TV}}(\tilde{\mu}_f\Vert\mu_f)[s]\right]-\mathop{\mathbb{E}}_{\substack{f\sim p_{\mathrm{eval}}(\cdot)\\s\sim d^{\mu_f}(\cdot)}}\left[D_{\mathrm{TV}}(\tilde{\mu}_f\Vert\mu_f)[s]\right]\right\vert\\
			&=\left\vert\mathop{\mathbb{E}}_{f\sim p_{\mathrm{eval}}(\cdot)}\left\{\mathop{\mathbb{E}}_{s\sim d^{\tilde{\mu}_f}(\cdot)}\left[D_{\mathrm{TV}}(\tilde{\mu}_f\Vert\mu_f)[s]\right]-\mathop{\mathbb{E}}_{s\sim d^{\mu_f}(\cdot)}\left[D_{\mathrm{TV}}(\tilde{\mu}_f\Vert\mu_f)[s]\right]\right\}\right\vert\\
			&\leq\mathop{\mathbb{E}}_{f\sim p_{\mathrm{eval}}(\cdot)}\left\{\left\vert\mathop{\mathbb{E}}_{s\sim d^{\tilde{\mu}_f}(\cdot)}\left[D_{\mathrm{TV}}(\tilde{\mu}_f\Vert\mu_f)[s]\right]-\mathop{\mathbb{E}}_{s\sim d^{\mu_f}(\cdot)}\left[D_{\mathrm{TV}}(\tilde{\mu}_f\Vert\mu_f)[s]\right]\right\vert\right\}\\
			&\leq\mathop{\mathbb{E}}_{f\sim p_{\mathrm{eval}}(\cdot)}\left\{\left\Vert d^{\tilde{\mu}_f}-d^{\mu_f}\right\Vert_1\cdot\left\Vert D_{\mathrm{TV}}(\tilde{\mu}_f\Vert\mu_f)[s]\right\Vert_{\infty}\right\}\\
			&\leq\mathop{\mathbb{E}}_{f\sim p_{\mathrm{eval}}(\cdot)}\left\{\frac{2\gamma}{1-\gamma}\mathop{\mathbb{E}}_{s\sim d^{\mu_f}(\cdot)}\left[D_{\mathrm{TV}}(\tilde{\mu}_f\Vert\mu_f)[s]\right]\cdot\max_{s}D_{\mathrm{TV}}(\tilde{\mu}_f\Vert\mu_f)[s]\right\}\\
			&\leq\frac{2\gamma\sigma_{\mathrm{eval}}}{1-\gamma}\mathop{\mathbb{E}}_{\substack{f\sim p_{\mathrm{eval}}(\cdot)\\s\sim d^{\mu_f}(\cdot)}}\left[D_{\mathrm{TV}}(\tilde{\mu}_f\Vert\mu_f)[s]\right]=\frac{2\gamma\sigma_{\mathrm{eval}}}{1-\gamma}\cdot\mathfrak{D}_{\mathrm{eval}},\\
		\end{split}
	\end{equation}
	as a result,
	\begin{equation}
		\mathfrak{D}_2\leq\left(1+\frac{2\gamma\sigma_{\mathrm{eval}}}{1-\gamma}\right)\mathfrak{D}_{\mathrm{eval}},
	\end{equation}
	concluding the proof.
\end{proof}

\subsection{Proof of Theorem \ref{theorem 5}}\label{proof 5}
\textbf{Theorem \ref{theorem 5}.} \textit{
	Given any two policies, $\tilde{\pi}$ and $\pi$, assume that $\tilde{\pi}$ is $\mathcal{R}_{\tilde{\pi}}$-robust, and $\pi$ is $\mathcal{R}_{\pi}$-robust, then the following bound holds:
	\begin{equation}
		\mathfrak{D}_{\mathrm{eval}}\leq\left(1+\frac{2\gamma\sigma_{\mathrm{train}}}{1-\gamma}\right)\mathcal{R}_{\pi}+\mathcal{R}_{\tilde{\pi}}+\mathfrak{D}_{\mathrm{train}}.
	\end{equation}
}
\begin{proof}
	Let's first rewrite $\mathfrak{D}_{\mathrm{eval}}$ as
	\begin{equation}
		\mathfrak{D}_{\mathrm{eval}}=\mathop{\mathbb{E}}_{\substack{\tilde{f}\sim p_{\mathrm{eval}}(\cdot)\\s\sim d^{\mu_{\tilde{f}}}(\cdot)}}\left[D_{\mathrm{TV}}(\tilde{\mu}_{\tilde{f}}\Vert\mu_{\tilde{f}})[s]\right].
	\end{equation}
	For another $f\in\mathcal{F}_{\mathrm{train}}$, by repeatedly using the triangle inequality of the TV distance, we have
	\begin{equation}
		\begin{split}
			\mathfrak{D}_{\mathrm{eval}}&=\mathop{\mathbb{E}}_{\substack{\tilde{f}\sim p_{\mathrm{eval}}(\cdot)\\s\sim d^{\mu_{\tilde{f}}}(\cdot)}}\left[D_{\mathrm{TV}}(\tilde{\mu}_{\tilde{f}}\Vert\mu_{\tilde{f}})[s]\right]\\
			&\leq\mathop{\mathbb{E}}_{\substack{\tilde{f}\sim p_{\mathrm{eval}}(\cdot)\\s\sim d^{\mu_{\tilde{f}}}(\cdot)}}\left[D_{\mathrm{TV}}(\tilde{\mu}_{\tilde{f}}\Vert\tilde{\mu}_{f})[s]+D_{\mathrm{TV}}(\tilde{\mu}_{f}\Vert\mu_f)[s]+D_{\mathrm{TV}}(\mu_f\Vert\mu_{\tilde{f}})[s]\right] \\
			&=\mathop{\mathbb{E}}_{\substack{\tilde{f}\sim p_{\mathrm{eval}}(\cdot)\\s\sim d^{\mu_{\tilde{f}}}(\cdot)}}\left[D_{\mathrm{TV}}(\tilde{\mu}_{\tilde{f}}\Vert\tilde{\mu}_{f})[s]\right]+\mathop{\mathbb{E}}_{\substack{\tilde{f}\sim p_{\mathrm{eval}}(\cdot)\\s\sim d^{\mu_{\tilde{f}}}(\cdot)}}\left[D_{\mathrm{TV}}(\tilde{\mu}_{f}\Vert\mu_f)[s]\right]+\mathop{\mathbb{E}}_{\substack{\tilde{f}\sim p_{\mathrm{eval}}(\cdot)\\s\sim d^{\mu_{\tilde{f}}}(\cdot)}}\left[D_{\mathrm{TV}}(\mu_f\Vert\mu_{\tilde{f}})[s]\right], \\
		\end{split}
	\end{equation}
	taking the expectation of both sides of the inequality with respect to $f\sim p_{\mathrm{train}}(\cdot)$, we obtain
	\begin{equation}
		\mathop{\mathbb{E}}_{f\sim p_{\mathrm{train}}(\cdot)}\left[\mathfrak{D}_{\mathrm{eval}}\right]\leq\mathop{\mathbb{E}}_{\substack{f\sim p_{\mathrm{train}}(\cdot)\\\tilde{f}\sim p_{\mathrm{eval}}(\cdot)\\s\sim d^{\mu_{\tilde{f}}}(\cdot)}}\left[D_{\mathrm{TV}}(\tilde{\mu}_{\tilde{f}}\Vert\tilde{\mu}_{f})[s]\right]+\mathop{\mathbb{E}}_{\substack{f\sim p_{\mathrm{train}}(\cdot)\\\tilde{f}\sim p_{\mathrm{eval}}(\cdot)\\s\sim d^{\mu_{\tilde{f}}}(\cdot)}}\left[D_{\mathrm{TV}}(\tilde{\mu}_{f}\Vert\mu_f)[s]\right]+\mathop{\mathbb{E}}_{\substack{f\sim p_{\mathrm{train}}(\cdot)\\\tilde{f}\sim p_{\mathrm{eval}}(\cdot)\\s\sim d^{\mu_{\tilde{f}}}(\cdot)}}\left[D_{\mathrm{TV}}(\mu_f\Vert\mu_{\tilde{f}})[s]\right].
	\end{equation}
	Since $\mathfrak{D}_{\mathrm{eval}}$ is independent of $f$, it becomes a constant after taking the expectation, which is
	\begin{equation}
		\mathfrak{D}_{\mathrm{eval}}\leq\mathop{\mathbb{E}}_{\substack{f\sim p_{\mathrm{train}}(\cdot)\\\tilde{f}\sim p_{\mathrm{eval}}(\cdot)\\s\sim d^{\mu_{\tilde{f}}}(\cdot)}}\left[D_{\mathrm{TV}}(\tilde{\mu}_{\tilde{f}}\Vert\tilde{\mu}_{f})[s]\right]+\mathop{\mathbb{E}}_{\substack{f\sim p_{\mathrm{train}}(\cdot)\\\tilde{f}\sim p_{\mathrm{eval}}(\cdot)\\s\sim d^{\mu_{\tilde{f}}}(\cdot)}}\left[D_{\mathrm{TV}}(\tilde{\mu}_{f}\Vert\mu_f)[s]\right]+\mathop{\mathbb{E}}_{\substack{f\sim p_{\mathrm{train}}(\cdot)\\\tilde{f}\sim p_{\mathrm{eval}}(\cdot)\\s\sim d^{\mu_{\tilde{f}}}(\cdot)}}\left[D_{\mathrm{TV}}(\mu_f\Vert\mu_{\tilde{f}})[s]\right].
	\end{equation}
	Note that $\tilde{\pi}$ is $\mathcal{R}_{\tilde{\pi}}$-robust, and $\pi$ is $\mathcal{R}_{\pi}$-robust, we can thus bound the first term:
	\begin{equation}\label{kl 1}
		\begin{split}
			\mathop{\mathbb{E}}_{\substack{f\sim p_{\mathrm{train}}(\cdot)\\\tilde{f}\sim p_{\mathrm{eval}}(\cdot)\\s\sim d^{\mu_{\tilde{f}}}(\cdot)}}\left[D_{\mathrm{TV}}(\tilde{\mu}_{\tilde{f}}\Vert\tilde{\mu}_{f})[s]\right]=&\mathop{\mathbb{E}}_{\substack{f\sim p_{\mathrm{train}}(\cdot)\\\tilde{f}\sim p_{\mathrm{eval}}(\cdot)}}\left[\sum_{s\in\mathcal{S}}d^{\mu_{\tilde{f}}}(s)\cdot D_{\mathrm{TV}}(\tilde{\mu}_{\tilde{f}}\Vert\tilde{\mu}_{f})[s]\right]\\
			\leq&\mathop{\mathbb{E}}_{\substack{f\sim p_{\mathrm{train}}(\cdot)\\\tilde{f}\sim p_{\mathrm{eval}}(\cdot)}}\left[\sum_{s\in\mathcal{S}}d^{\mu_{\tilde{f}}}(s)\cdot\mathcal{R}_{\tilde{\pi}}\right]=\mathcal{R}_{\tilde{\pi}}\mathop{\mathbb{E}}_{\substack{f\sim p_{\mathrm{train}}(\cdot)\\\tilde{f}\sim p_{\mathrm{eval}}(\cdot)}}\left[\sum_{s\in\mathcal{S}}d^{\mu_{\tilde{f}}}(s)\right]=\mathcal{R}_{\tilde{\pi}}. \\
		\end{split}
	\end{equation}
	Similarly, we can bound the third term:
	\begin{equation} \label{kl 3}
		\begin{split}
			\mathop{\mathbb{E}}_{\substack{f\sim p_{\mathrm{train}}(\cdot)\\\tilde{f}\sim p_{\mathrm{eval}}(\cdot)\\s\sim d^{\mu_{\tilde{f}}}(\cdot)}}\left[D_{\mathrm{TV}}(\mu_{\tilde{f}}\Vert\mu_{f})[s]\right]=&\mathop{\mathbb{E}}_{\substack{f\sim p_{\mathrm{train}}(\cdot)\\\tilde{f}\sim p_{\mathrm{eval}}(\cdot)}}\left[\sum_{s\in\mathcal{S}}d^{\mu_{\tilde{f}}}(s)\cdot D_{\mathrm{TV}}(\mu_{\tilde{f}}\Vert\mu_{f})[s]\right]\\
			\leq&\mathop{\mathbb{E}}_{\substack{f\sim p_{\mathrm{train}}(\cdot)\\\tilde{f}\sim p_{\mathrm{eval}}(\cdot)}}\left[\sum_{s\in\mathcal{S}}d^{\mu_{\tilde{f}}}(s)\cdot\mathcal{R}_{\pi}\right]=\mathcal{R}_{\pi}\mathop{\mathbb{E}}_{\substack{f\sim p_{\mathrm{train}}(\cdot)\\\tilde{f}\sim p_{\mathrm{eval}}(\cdot)}}\left[\sum_{s\in\mathcal{S}}d^{\mu_{\tilde{f}}}(s)\right]=\mathcal{R}_{\pi}. \\
		\end{split}
	\end{equation}
	Next, we are trying to bound the second term, which is similar to $\mathfrak{D}_{\mathrm{train}}$. Note that $\mathfrak{D}_{\mathrm{train}}$ is independent of $\tilde{f}$, we can thus rewrite it as
	\begin{equation}
		\mathfrak{D}_{\mathrm{train}}=\mathop{\mathbb{E}}_{\substack{f\sim p_{\mathrm{train}}(\cdot)\\s\sim d^{\mu_f}(\cdot)}}\left[D_{\mathrm{TV}}(\tilde{\mu}_f\Vert\mu_f)[s]\right]=\mathop{\mathbb{E}}_{\substack{f\sim p_{\mathrm{train}}(\cdot)\\\tilde{f}\sim p_{\mathrm{eval}}(\cdot)\\s\sim d^{\mu_f}(\cdot)}}\left[D_{\mathrm{TV}}(\tilde{\mu}_{f}\Vert\mu_f)[s]\right],
	\end{equation}
	then
	\begin{equation}
		\begin{split}
			&\left\vert\mathop{\mathbb{E}}_{\substack{f\sim p_{\mathrm{train}}(\cdot)\\\tilde{f}\sim p_{\mathrm{eval}}(\cdot)\\s\sim d^{\mu_{\tilde{f}}}(\cdot)}}\left[D_{\mathrm{TV}}(\tilde{\mu}_{f}\Vert\mu_f)[s]\right]-\mathfrak{D}_{\mathrm{train}}\right\vert\\
			=&\left\vert\mathop{\mathbb{E}}_{\substack{f\sim p_{\mathrm{train}}(\cdot)\\\tilde{f}\sim p_{\mathrm{eval}}(\cdot)\\s\sim d^{\mu_{\tilde{f}}}(\cdot)}}\left[D_{\mathrm{TV}}(\tilde{\mu}_{f}\Vert\mu_f)[s]\right]-\mathop{\mathbb{E}}_{\substack{f\sim p_{\mathrm{train}}(\cdot)\\\tilde{f}\sim p_{\mathrm{eval}}(\cdot)\\s\sim d^{\mu_f}(\cdot)}}\left[D_{\mathrm{TV}}(\tilde{\mu}_{f}\Vert\mu_f)[s]\right]\right\vert \\
			=&\left\vert\int_{\mathcal{F}_{\mathrm{train}}}p_{\mathrm{train}}(f)\int_{\mathcal{F}_{\mathrm{eval}}}p_{\mathrm{eval}}(\tilde{f})\left\{\mathop{\mathbb{E}}_{s\sim d^{\mu_{\tilde{f}}}(\cdot)}\left[D_{\mathrm{TV}}(\tilde{\mu}_{f}\Vert\mu_f)[s]\right]-\mathop{\mathbb{E}}_{s\sim d^{\mu_f}(\cdot)}\left[D_{\mathrm{TV}}(\tilde{\mu}_{f}\Vert\mu_f)[s]\right]\right\}\mathrm{d}\tilde{f}\mathrm{d}f\right\vert \\
			\leq&\int_{\mathcal{F}_{\mathrm{train}}}p_{\mathrm{train}}(f)\int_{\mathcal{F}_{\mathrm{eval}}}p_{\mathrm{eval}}(\tilde{f})\left\{\left\vert\mathop{\mathbb{E}}_{s\sim d^{\mu_{\tilde{f}}}(\cdot)}\left[D_{\mathrm{TV}}(\tilde{\mu}_{f}\Vert\mu_f)[s]\right]-\mathop{\mathbb{E}}_{s\sim d^{\mu_f}(\cdot)}\left[D_{\mathrm{TV}}(\tilde{\mu}_{f}\Vert\mu_f)[s]\right]\right\vert\right\}\mathrm{d}\tilde{f}\mathrm{d}f. \\
		\end{split}
	\end{equation}
	Note that,
	\begin{equation}
		\left\vert\mathop{\mathbb{E}}_{s\sim d^{\mu_{\tilde{f}}}(\cdot)}\left[D_{\mathrm{TV}}(\tilde{\mu}_{f}\Vert\mu_f)[s]\right]-\mathop{\mathbb{E}}_{s\sim d^{\mu_f}(\cdot)}\left[D_{\mathrm{TV}}(\tilde{\mu}_{f}\Vert\mu_f)[s]\right]\right\vert\leq\left\Vert d^{\mu_{\tilde{f}}}-d^{\mu_f}\right\Vert_1\cdot\left\Vert D_{\mathrm{TV}}(\tilde{\mu}_{f}\Vert\mu_f)[s]\right\Vert_{\infty}.
	\end{equation}
	According to Lemma \ref{lemma 2},
	\begin{equation}
		\left\Vert d^{\mu_{\tilde{f}}}-d^{\mu_f}\right\Vert_1\leq\frac{2\gamma}{1-\gamma}\mathop{\mathbb{E}}_{s\sim d^{\mu_f}(\cdot)}\left[D_{\mathrm{TV}}(\mu_{\tilde{f}}\Vert\mu_f)[s]\right],
	\end{equation}
	$\pi$ is $\mathcal{R}_{\pi}$-robust, so,
	\begin{equation}
		\left\Vert d^{\mu_{\tilde{f}}}-d^{\mu_f}\right\Vert_1\leq\frac{2\gamma}{1-\gamma}\mathop{\mathbb{E}}_{s\sim d^{\mu_f}(\cdot)}\left[D_{\mathrm{TV}}(\mu_{\tilde{f}}\Vert\mu_f)[s]\right]=\frac{2\gamma}{1-\gamma}\sum_{s\in\mathcal{S}}d^{\mu_{f}}(s)\cdot D_{\mathrm{TV}}(\mu_{\tilde{f}}\Vert\mu_f)[s]\leq\frac{2\gamma}{1-\gamma}\mathcal{R}_{\pi}.
	\end{equation}
	As a result,
	\begin{equation}
		\begin{split}
			&\left\vert\mathop{\mathbb{E}}_{\substack{f\sim p_{\mathrm{train}}(\cdot)\\\tilde{f}\sim p_{\mathrm{eval}}(\cdot)\\s\sim d^{\mu_{\tilde{f}}}(\cdot)}}\left[D_{\mathrm{TV}}(\tilde{\mu}_{f}\Vert\mu_f)[s]\right]-\mathfrak{D}_{\mathrm{train}}\right\vert\\
			\leq&\int_{\mathcal{F}_{\mathrm{train}}}p_{\mathrm{train}}(f)\int_{\mathcal{F}_{\mathrm{eval}}}p_{\mathrm{eval}}(\tilde{f})\cdot\left\{\left\vert\mathop{\mathbb{E}}_{s\sim d^{\mu_{\tilde{f}}}(\cdot)}\left[D_{\mathrm{TV}}(\tilde{\mu}_{f}\Vert\mu_f)[s]\right]-\mathop{\mathbb{E}}_{s\sim d^{\mu_f}(\cdot)}\left[D_{\mathrm{TV}}(\tilde{\mu}_{f}\Vert\mu_f)[s]\right]\right\vert\right\}\mathrm{d}\tilde{f}\mathrm{d}f \\
			\leq&\int_{\mathcal{F}_{\mathrm{train}}}p_{\mathrm{train}}(f)\int_{\mathcal{F}_{\mathrm{eval}}}p_{\mathrm{eval}}(\tilde{f})\cdot\left\{\frac{2\gamma}{1-\gamma}\mathcal{R}_{\pi}\cdot\max_{s}D_{\mathrm{TV}}(\tilde{\mu}_{f}\Vert\mu_f)[s]\right\}\mathrm{d}\tilde{f}\mathrm{d}f \\
			=&\int_{\mathcal{F}_{\mathrm{train}}}p_{\mathrm{train}}(f)\cdot\left\{\frac{2\gamma}{1-\gamma}\mathcal{R}_{\pi}\cdot\max_{s}D_{\mathrm{TV}}(\tilde{\mu}_{f}\Vert\mu_f)[s]\right\}\cdot\int_{\mathcal{F}_{\mathrm{eval}}}p_{\mathrm{eval}}(\tilde{f})\mathrm{d}\tilde{f}\mathrm{d}f \\
			=&\int_{\mathcal{F}_{\mathrm{train}}}p_{\mathrm{train}}(f)\cdot\left\{\frac{2\gamma}{1-\gamma}\mathcal{R}_{\pi}\cdot\max_{s}D_{\mathrm{TV}}(\tilde{\mu}_{f}\Vert\mu_f)[s]\right\}\mathrm{d}f=\frac{2\gamma}{1-\gamma}\mathcal{R}_{\pi}\int_{\mathcal{F}_{\mathrm{train}}}p_{\mathrm{train}}(f)\cdot\max_{s}D_{\mathrm{TV}}(\tilde{\mu}_{f}\Vert\mu_f)[s]\mathrm{d}f. \\
		\end{split}
	\end{equation}
	We previously defined $\sigma_{\mathrm{train}}=\max_{f\in\mathcal{F}_{\mathrm{train}}}\left\{\max_{s}D_{\mathrm{TV}}(\tilde{\mu}_f\Vert\mu_f)[s]\right\}$, so that
	\begin{equation}
		\begin{split}
			\left\vert\mathop{\mathbb{E}}_{\substack{f\sim p_{\mathrm{train}}(\cdot)\\\tilde{f}\sim p_{\mathrm{eval}}(\cdot)\\s\sim d^{\mu_{\tilde{f}}}(\cdot)}}\left[D_{\mathrm{TV}}(\tilde{\mu}_{f}\Vert\mu_f)[s]\right]-\mathfrak{D}_{\mathrm{train}}\right\vert\leq&\frac{2\gamma}{1-\gamma}\mathcal{R}_{\pi}\int_{\mathcal{F}_{\mathrm{train}}}p_{\mathrm{train}}(f)\cdot\max_{s}D_{\mathrm{TV}}(\tilde{\mu}_{f}\Vert\mu_f)[s]\mathrm{d}f \\
			\leq&\frac{2\gamma\sigma_{\mathrm{train}}}{1-\gamma}\mathcal{R}_{\pi}\int_{\mathcal{F}_{\mathrm{train}}}p_{\mathrm{train}}(f)\mathrm{d}f=\frac{2\gamma\sigma_{\mathrm{train}}}{1-\gamma}\mathcal{R}_{\pi}, \\
		\end{split}
	\end{equation}
	thus, the second term is bounded by
	\begin{equation}\label{kl 2}
		\mathop{\mathbb{E}}_{\substack{f\sim p_{\mathrm{train}}(\cdot)\\\tilde{f}\sim p_{\mathrm{eval}}(\cdot)\\s\sim d^{\mu_{\tilde{f}}}(\cdot)}}\left[D_{\mathrm{TV}}(\tilde{\mu}_{f}\Vert\mu_f)[s]\right]\leq\frac{2\gamma\sigma_{\mathrm{train}}}{1-\gamma}\mathcal{R}_{\pi}+\mathfrak{D}_{\mathrm{train}}.
	\end{equation}
	Finally, combining (\ref{kl 1}), (\ref{kl 3}) and (\ref{kl 2}), we have
	\begin{equation}
		\mathfrak{D}_{\mathrm{eval}}\leq\left(1+\frac{2\gamma\sigma_{\mathrm{train}}}{1-\gamma}\right)\mathcal{R}_{\pi}+\mathcal{R}_{\tilde{\pi}}+\mathfrak{D}_{\mathrm{train}},
	\end{equation}
	concluding the proof.
\end{proof}

\end{document}